\documentclass{article} 
\usepackage{collas2025_conference,times}
\usepackage{easyReview}

\usepackage{hyperref}
\hypersetup{
    colorlinks=true,
    linkcolor=red,
    filecolor=magenta,
    urlcolor=blue,
    citecolor=purple,
    pdftitle={Overleaf Example},
    pdfpagemode=FullScreen,
    }

\usepackage{url}            
\usepackage{nicefrac}       

\usepackage{amsmath}
\usepackage{amssymb}
\usepackage{mathtools}
\usepackage{amsthm}
\usepackage{bbm}

\theoremstyle{plain}
\newtheorem{theorem}{Theorem}[section]

\newtheorem{lemma}[theorem]{Lemma}
\newtheorem{corollary}[theorem]{Corollary}
\theoremstyle{definition}
\newcounter{assume}

\newcounter{define}

\newtheorem{definition}[define]{Definition}
\newtheorem{assumption}[assume]{Assumption}
\theoremstyle{remark}

\usepackage{graphicx} 

\usepackage{algorithm}
\usepackage{algorithmic}

\usepackage{multirow}
\usepackage{ctable}
\usepackage{subfigure}
\def\argmin{\mathop{\mathrm{argmin}}}

\title{Data-dependent and Oracle Bounds on Forgetting in Continual Learning}

 \author{Lior Friedman \\
Department of Electrical and Computer Engineering\\
Technion - Israel Institute of Technology\\
Haifa 3200003, Israel \\
\texttt{liorf@campus.technion.ac.il} \\
\And 
Ron Meir  \\
Department of Electrical and Computer Engineering \\
Technion - Israel Institute of Technology \\
Haifa 3200003, Israel \\
\texttt{rmeir@ee.technion.ac.il} \\
}

\collasfinalcopy 

\begin{document}

\maketitle

\begin{abstract}
In continual learning, knowledge must be preserved and re-used between tasks, maintaining good transfer to future tasks and minimizing forgetting of previously learned ones. While several practical algorithms have been devised for this setting, there have been few theoretical works aiming to quantify and bound the degree of Forgetting in general settings. For \emph{exemplar-free} methods, we provide both data-dependent upper bounds that apply \emph{regardless of model and algorithm choice}, and oracle bounds for Gibbs posteriors. We derive an algorithm based on our bounds and demonstrate empirically that our approach yields tight and practical bounds on forgetting for several continual learning problems and algorithms. 
\end{abstract}

\section{Introduction}

Continual learning is a burgeoning machine learning setting where data from different tasks
are presented sequentially to the learner. The usual stated goal of methods in this setting is to adapt the learner to new tasks as they appear while also preserving its performance on previous tasks \citep{de2021continual, hadsell2020embracing,parisi2019continual}.
This performance on previous tasks is called backward transfer, or forgetting, and one of the key challenges in continual learning is avoiding \emph{Catastrophic Forgetting} \citep{goodfellow2015empirical,ramasesh2020anatomy,kirkpatrick2017overcoming}, where performance on previous tasks degrades significantly as the model adapts to new tasks. 
Although avoiding catastrophic forgetting is desirable, much of the focus of continual learning research in recent years was on settings with a shared optimal solution. Realistically, 
we should not expect one model to perform optimally across all settings, for example if the data distribution changes gradually. A recent paper by \citet{kumar2023continual} discusses continual learning
as a computationally constrained optimization problem and argues that forgetting non-recurring
information is not “catastrophic", especially in changing environments. We will discuss this topic in our empirical evaluation.

While there have been several empirical approaches in the field, there are relatively few theoretical works that explore and attempt to quantify and bound this backward transfer. Some, such as \citet{evron2022catastrophic, lin2023theory, li2024theory} focus on linear regression models to consider the effect of task order and similarity on forgetting. Others, such as \citet{bennani2020generalisation, doan2021theoretical} utilize the Neural Tangent Kernel (NTK) regime to focus on more complex task similarity measures as predictors of forgetting.
Several more general works such as \citet{benavides2022theory} apply notions of VC-dimension to arrive at more general scaling laws and data-agnostic upper bounds on forgetting for the ERM model, but may be difficult to apply for complex predictors due to the potentially large VC-dimension of models such as deep neural networks \citep{bartlett2019nearly}. 
We note that many of the known results \citep{bennani2020generalisation, evron2022catastrophic,  benavides2022theory} provide upper bound on forgetting for the training data, though there are previous results on forgetting for test data as well \citep{doan2021theoretical, lin2023theory, li2024theory} for specific models. In particular, \cite{lin2023theory} provide an exact characterization of forgetting for the Stochastic Gradient Descent (SGD) algorithm under specific assumptions in the linear regression setting.
\citet{bennani2020generalisation, doan2021theoretical} derive bounds for more general settings for very wide architectures that follow the NTK regime, providing bounds specifically for the SGD algorithm as well as a continual learning variant known as Orthogonal Gradient Descent (OGD).

In this work, we explore upper bounds on forgetting that apply to general architectures and optimization algorithms
. We note that existing results are restricted to specific setups (e.g.linear regression) and algorithms (e.g. SGD). We use the PAC-Bayes \citep{Mcallester, Catoni2004, alquier2021user}  
framework to derive and analyze upper bounds on backward transfer, focusing on the Gibbs posterior \citep{casella1992explaining}. In Section \ref{sec:data-dep-bounds} we derive general bounds for the two-task and multiple-task settings without making any model assumptions. We then focus our discussion in Section \ref{sec:oracle-bounds} on oracle bounds for the Gibbs posterior in general, and under specific assumptions of task similarity. 
We then proceed to empirically demonstrate our bounds\footnote{Code is available at \url{https://github.com/lioritan/continual_forgetting_pb}} for both a Variational Inference (VI) \citep{hoffman2013stochastic} based algorithm as well as Elastic Weight Consolidation (EWC) \citep{kirkpatrick2017overcoming}, on several continual learning problems in Section \ref{sec:empirical}. To the best of our knowledge, our bounds are the first general PAC-Bayes upper bounds on forgetting in continual learning that are applicable regardless of model choice and task distribution, as well as the first established forgetting upper bounds for the EWC algorithm. Table \ref{settings-table} provides an overview of existing settings and limitations for known results, as well as our own results.

\begin{table}[h]
\caption{Settings covered by existing test forgetting bounds as well as our bounds (final two rows). Methods marked with * use added noise during bound estimation for deterministic architectures.} 
\label{settings-table}
\vskip 0.1in
\begin{center}
\begin{small}
\begin{sc}
\begin{tabular}{lccc}
Algorithm  &Architecture & Task construction & Reference \\
\midrule
SGD & Linear, overparametrized & Linear regression & \cite{lin2023theory}\\
SGD & Linear, Mixture-of-experts & Linear regression & \cite{li2024theory}\\
SGD and OGD & NTK regime & Any & \cite{bennani2020generalisation}\\
SGD and OGD & NTK regime & Any & \cite{doan2021theoretical}\\
Data-Agnostic & Any & $\mathcal{F}-$related tasks & \cite{benavides2022theory}\\
Any & Any* & Any & Corollary \ref{thm:forget-sumT-final}\\
Emprirical Gibbs & Any* & Any & Theorems \ref{thm:forgetting-asymptotic},\ref{thm:forgetting-asymptotic-cov}\\
\end{tabular}
\end{sc}
\end{small}
\end{center}
\vskip -0.1in
\end{table}

\section{Problem definition}
\begin{figure}[h]
    \centering
    \includegraphics[width=0.4\linewidth]{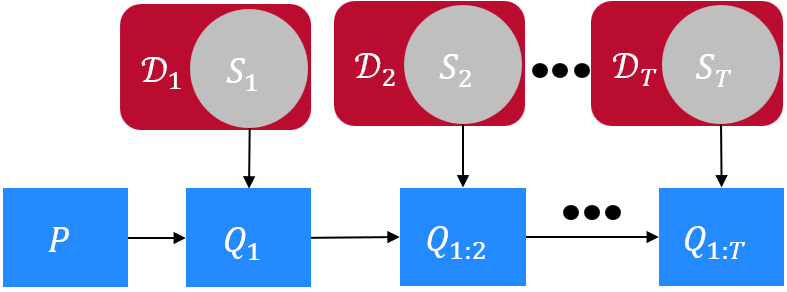}
    \caption{Continual Learning problem. A data-free prior $P$ is adapted via an empirical sample $S_1\sim\mathcal{D}_1$ to a posterior $Q_1$. This posterior serves as the new prior for the next task and so on until we reach a final posterior $Q_{1:T}$.}
    \label{fig:setup-label}
\end{figure}
We consider a finite sequence of tasks $\{1,2,\ldots,T\}\triangleq [T]$, where for each task $k\in[T]$, we are given a batch of i.i.d.~data $S_k\sim \mathcal{D}_k$. 
The sample for a given task is defined as $S=\{z_i\}_{i=1}^m, z_i=(x_i,y_i)$ where $x_i\in \mathcal{X}, y_i\in \mathcal{Y}$.
 A hypothesis $h\in \mathcal{H}$ is a mapping $h:\mathcal{X}\rightarrow \mathcal{Y}$  characterized by a loss $\ell(h,z)$.

\begin{definition}
	The expected loss of a given hypothesis $h\in \mathcal{H}$ is defined as $\mathcal{L}(h, \mathcal{D}) \triangleq \mathbb{E}_{z\in \mathcal{D}} \ell(h, z)$. The empirical loss of a hypothesis w.r.t.~a sample $S$ is defined as $\hat{\mathcal{L}}(h, S) \triangleq \frac{1}{m}\sum_{j=1}^{m}\ell(h, z_j)$.
\end{definition}

In the following Section, and in Section \ref{sec:oracle-bounds}, we first consider only two tasks from distributions $\mathcal{D}_s$ and $\mathcal{D}_t$, referring to source and target. Within the classic PAC-Bayes setting \citep{Mcallester,alquier2021user}, let $Q_s$ be a distribution over the set of hypotheses learned by some algorithm $J_s$ over $S_s\sim \mathcal{D}_s$ and a data-free prior hypothesis distribution $P$, such that $Q_s=J_s(S_s, P).$ We then proceed to utilize another algorithm $J_t$ that operates on $S_t\sim \mathcal{D}_t$, making use of $Q_s$, such that 
$Q_{s:t}=J_t(S_t, Q_s).$  
For now, we make no assumptions on $J_s$ and $S_t$ other than their inputs. Of particular note is that algorithm $J_t$ is exemplar-free (i.e. it has no access to data from $\mathcal{D}_s$), and that task $\mathcal{D}_t$ is chosen independently of the prior $Q_s$, meaning that task choice is independent of the learning process.

\begin{definition}
	The \emph{backwards transfer loss} of $Q_{s:t}$ on task $\mathcal{D}_s$ is defined as $$\mathrm{BWT}(Q_{s:t}, \mathcal{D}_s) \triangleq \mathbb{E}_{h\sim Q_{s:t}}\left [\mathcal{L}(h, \mathcal{D}_s)\right ]=\mathcal{L}(Q_{s:t}, \mathcal{D}_s).$$
	The \emph{forgetting} of $Q_{s:t}$ on task $\mathcal{D}_s$ is defined as \begin{align*}
 F(Q_{s:t}, \mathcal{D}_s) &\triangleq \mathrm{BWT}(Q_{s:t}, \mathcal{D}_s) - \mathbb{E}_{h\sim Q_{s}}\left [\mathcal{L}(h, \mathcal{D}_s)\right ]=\mathcal{L}(Q_{s:t}, \mathcal{D}_s)-\mathcal{L}(Q_{s}, \mathcal{D}_s).\end{align*}
\end{definition}
Intuitively, the backward transfer $\mathcal{L}(Q_{s:t}, \mathcal{D}_s)$ measures the performance of the updated model on the previous task $\mathcal{D}_s$, and the forgetting $F(Q_{s:t}, \mathcal{D}_s)$ measures how much worse this performance is compared to the loss immediately after learning $S_s\sim\mathcal{D}_s$. Since the continual learning setting assumes that tasks are given in order, when we are given task $\mathcal{D}_t$ we can no longer optimize $\mathcal{L}(Q_{s}, \mathcal{D}_s)$. As such, from an optimization perspective, the best we can do at any given point is to minimize the backward transfer loss $\mathcal{L}(Q_{s:t}, \mathcal{D}_s)$. Our definition of forgetting is identical to that of \citet{lin2023theory} (but for non-deterministic hypothesis). Our definition of backward transfer could be equated to either the positive part of the forgetting loss or to their definition of ``generalization error", taken only on previous tasks.

We note that this definition refers to the \emph{test forgetting}, meaning the model's ability to still generalize well on previously seen domains, rather than measuring the retention of training performance on previous tasks. Due to this definition, simple measures such as memorization of previous training tasks cannot effectively minimize forgetting.
\begin{definition}
    The transfer loss of $Q_{s:t}$ on task $\mathcal{D}_t$ is defined as $\mathcal{L}(Q_{s:t}, \mathcal{D}_t)$, also known as the generalization loss.
\end{definition}

Compared to the backward transfer, forward transfer (i.e., generalization) is better explored in general, and in the continual learning setting in particular, and several bounds are available \citep{bennani2020generalisation, benavides2022theory}. 

We also note that several existing PAC-Bayes bounds for generalization can be extended or altered for forward transfer in the continual setting:  \citet{haddouche2022online} derive generalization bounds for online learning, and both \citet{haddouche2023pac} and \citet{chugg2023unified} derive anytime bounds that can be modified for generalization in the continual learning setting. While not directly relevant, we also note that there are existing PAC-Bayes bounds for generalization in the context of meta-learning \citep{pentina2014pac,pentina2015lifelong,amit2018meta}, that can be considered as an offline version of the continual learning setting. To the best of our knowledge, there are no bounds on forgetting in the context of meta-learning.

Extending these notions to multiple task setting, where tasks $\{1,\ldots,T\}$ appear sequentially,  we define 
 \begin{align*}
 &BWT(Q_{1:T}, [T]) \triangleq \frac{1}{T-1}\sum_{t=1}^{T-1}\mathcal{L}(Q_{1:T}, \mathcal{D}_t)\\
 &F(Q_{1:T}, [T]) \triangleq \frac{1}{T-1}\sum_{t=1}^{T-1}\left (\mathcal{L}(Q_{1:T}, \mathcal{D}_t)-\mathcal{L}(Q_{1:t}, \mathcal{D}_t)\right )\end{align*}
where $Q_{1:t}$ is the learned posterior based on $t$ tasks. See also Appendix \ref{append:notation} for a detailed summary of notation.
%

\section{Data-dependent bounds for backward transfer}\label{sec:data-dep-bounds}

As we mentioned previously, in the continual learning setting we have no access to future tasks, and thus minimizing test forgetting reduces to minimizing the overall backward transfer as a proxy objective. Since any upper bound on backward transfer can be converted to an upper bound on forgetting, we focus our efforts on upper bounds for backward transfer.
In order to arrive at an upper bound on the backward transfer, we make use of concentration inequalities, relying on the change-of-measure inequality of \citet{donsker1975large}.
%
%
\begin{theorem} \label{thm:forget-base2} $\mathrm{(Forgetting)}$
    For any fixed $S_s,S_t,Q_s,Q_{s:t}$, and $
    \lambda_t>0$, 
    \begin{align}\begin{split}
 \label{eq:forget-base2}
&\mathcal{L}(Q_{s:t}, \mathcal{D}_s) \leq \hat{\mathcal{L}}(Q_{s:t}, S_t) + \frac{1}{\lambda_t} D_{\mathrm{KL}}(Q_{s:t}||Q_{s})
+\frac{1}{\lambda_t}\log\mathbb{E}_{h\sim Q_{s}}\left [e^{\lambda_t(\mathcal{L}(h,\mathcal{D}_s)-\hat{\mathcal{L}}(h,S_t))} \right ].
\end{split}\end{align}
\end{theorem}
The full proof of Theorem \ref{thm:forget-base2} is provided in Appendix \ref{append:proofs}. 
The proof follows the same basic techniques and structure of PAC-Bayes generalization upper bounds \citep{Mcallester,alquier2021user}.
The main idea is to use Lemma \ref{lemma:concentration} (the change-of-measure inequality) with $f(z)=\lambda_t(\mathcal{L}(z,\mathcal{D}_s)-\hat{\mathcal{L}}(z,S_t))$.
Unlike standard (forward) transfer results, the l.h.s.\ depends on the source task $s$ while the r.h.s.\ depends on task $t$. This is necessary since we no longer have access to the training data of the source task $S_s$. By subtracting $\mathcal{L}(Q_s,\mathcal{D}_s)$ from both sides, we arrive at an upper bound on the forgetting.

Looking at the three terms in the r.h.s.~of \eqref{eq:forget-base2} we see a complex tradeoff:
if $\hat{\mathcal{L}}(Q_s,S_t)$ is small, the first two terms will be small as choosing $Q_{s:t}=Q_s$ or a posterior with low KL-divergence will lead to low empirical loss, but the final term will be larger (we subtract a smaller number in the exponent). If, on the other hand, $\hat{\mathcal{L}}(Q_s,S_t)$ is high, we can expect the disagreement term  $$\frac{1}{\lambda_t}\log\mathbb{E}_{h\sim Q_{s}}\left [e^{\lambda_t(\mathcal{L}(h,\mathcal{D}_s)-\hat{\mathcal{L}}(h,S_t))} \right ]$$ to be small, but the first two terms will increase as achieving low $\hat{\mathcal{L}}(Q_{s:t},S_t)$ requires higher KL-divergence. In both cases, if the prior $Q_s$ transfers well to the new task $S_t$, as expected for similar tasks, we see a lower r.h.s.

We note that unlike the other terms in the r.h.s.~of \eqref{eq:forget-base2}, the expected loss on the previous task $\mathcal{L}(h,\mathcal{D}_s)$ is not data-dependent. However, since the disagreement term does not depend on the current posterior $Q_{s:t}$, it has no direct effect on any optimization of the posterior. 

Throughout the rest of the paper, we focus on classification problems, though our bounds hold for any setting where the conditions hold. Similarly to many PAC-Bayes bounds, results can be extended to unbounded losses (e.g. regression) under moment conditions (e.g. sub-Gaussian losses \citep{alquier2016properties, alquier2021user}).
\begin{assumption}\label{assume:bounded-loss}
For any hypothesis and data, the loss is bounded, $\ell(h,z)\in [0, K].$
\end{assumption}

For the $T$-task setting, we can bound the individual losses for each task.
By using the same change of measure inequality as in Theorem \ref{thm:forget-base2}, and averaging over previous tasks, we get
\begin{corollary}\label{thm:forget-sumT-final} ($\mathrm{Forgetting}$)
Let $$\Delta \mathcal{L}(h,s,t)\triangleq \mathcal{L}(h,\mathcal{D}_s)-\mathcal{L}(h,\mathcal{D}_t).$$
For any fixed $\lambda_T>0$ and $Q_{1:T-1},S_1,\ldots,S_{T-1}$, the following applies uniformly for all posteriors $Q_{1:T}$ with probability at least $1-\delta$ over the choice of $S_T$ ($m_t=|S_t|$), 
    \begin{align}
    \label{eq:final-forget-bound}
F(Q_{1:T}, [T]) &\leq \hat{\mathcal{L}}(Q_{1:T}, S_T)-\frac{1}{T-1}\sum_{t=1}^{T-1}\mathcal{L}(Q_{1:t},\mathcal{D}_t)\nonumber + \frac{1}{\lambda_T} D_{\mathrm{KL}}(Q_{1:T}||Q_{1:T-1})\\&+\frac{\lambda_T K^2}{8m_T}+\frac{\log{1/\delta}}{\lambda_T}
 +\frac{1}{(T-1)\lambda_T}\sum_{t=1}^{T-1}\log\mathbb{E}_{Q_{1:T-1}}\left [e^{\lambda_T\Delta \mathcal{L}(h,t,T
 )} \right ].
\end{align}

\end{corollary}

The proof of Corollary \ref{thm:forget-sumT-final} appears in Appendix \ref{append:proofs}.

\section{Oracle bounds for backward transfer}\label{sec:oracle-bounds}

First, recall the Gibbs posterior in the single task setting.
\begin{definition}
    The empirical \emph{Gibbs posterior} with parameter $\lambda$, is defined as 
\begin{equation} \label{defn:gibbs}
\hat Q^\lambda_s(h)=\frac{P(h)e^{-\lambda\hat{\mathcal{L}}(h,S_s)}}{\mathbb{E}_{h\sim P}\left [e^{\lambda\hat{\mathcal{L}}(h,S_s)} \right ]}~.
\end{equation}
\end{definition}
While the general upper bounds on backward transfer in Section \ref{sec:data-dep-bounds} are useful in practice, as well as for designing theoretically motivated algorithms (see Section \ref{sec:empirical}), there is merit in trying to better understand the behavior of these bounds for specific posterior distributions. To that end, we consider bounds on performance relative to that of an oracle who \emph{knows} the data-distribution, as opposed to the data-dependent bounds established in Section \ref{sec:data-dep-bounds}. Specifically, we consider the continual Gibbs learner
$
\hat{Q}^{\lambda_t}_{s:t}(h)=\frac{Q_s(h)e^{-\lambda_t\hat{\mathcal{L}}(h,S_t)}}{\mathbb{E}_{h\sim Q_s}\left [e^{-\lambda_t\hat{\mathcal{L}}(h,S_t)}\right ]}~. 
$
The Gibbs learner is of particular interest in the context of analyzing bounds with KL-divergence as it provides an explicit expression for the divergence for any prior, as well as being the optimal posterior for minimizing the r.h.s. in Theorem \ref{thm:forget-base2}. In addition, we note that if $\lambda\rightarrow\infty$, the Gibbs posterior converges to the empirical risk minimizer (ERM), meaning we can potentially apply our upper bounds to more traditional methods such as SGD.

\begin{corollary}
 \label{thm:oracle-base}
For any $Q_s, S_s, \lambda_t>0$, 
\begin{align}\begin{split} \label{eq:oracle-base}
\mathbb{E}_{S_t\sim \mathcal{D}_t}&\mathcal{L}( \hat{Q}^{\lambda_t}_{s:t},\mathcal{D}_s)\leq \inf_{Q_{s:t}}\left \{ \mathcal{L}(Q_{s:t},\mathcal{D}_t) + \frac{1}{\lambda_t}D_{\mathrm{KL}}(Q_{s:t}||Q_{s}) \right \} +\frac{\lambda_t K^2}{8m_t}+\frac{1}{\lambda_t}\log\mathbb{E}_{h\sim Q_s}\left [e^{\lambda_t\Delta \mathcal{L}(h,s,t)} \right ].
\end{split}\end{align}
\end{corollary}
Proof of Corollary \ref{thm:oracle-base} is provided in Appendix \ref{append:proofs}. 
The main idea of the proof is to choose the Gibbs learner as the posterior in \eqref{eq:forget-base2}.
%
%
As we can see from Corollary \ref{thm:oracle-base}, the behavior of $Q_s$ with regard to both tasks can affect the bound significantly, though its exact effect is unclear.

In order to better understand the overall behavior of the bound's elements, as well as examining whether stronger assumptions can offer more insightful bounds, we define a measure of the task similarity based on the \emph{loss covariance} between the tasks:
$$
\mathrm{cov}_{\lambda_t}(P,s,t)\triangleq \mathrm{cov}_{h\sim P}\left (e^{-\lambda_t\hat{\mathcal{L}}(h,S_s)}, e^{-\lambda_t\hat{\mathcal{L}}(h,S_t)}\right ).
$$
%
%
We note that this covariance term is bounded in $[-1, 1]$. 
\begin{corollary} \label{thm:cov-2task}
For any $\lambda_t>0$, if \eqref{defn:gibbs} holds, and $\mathrm{cov}_{\lambda_t}(P,s,t)\geq 0$, 
%
\begin{align*} 
\mathbb{E}_{S_s,S_t\sim \mathcal{D}_s,\mathcal{D}_t}\mathcal{L}( \hat{Q}^{\lambda_t}_{s:t},\mathcal{D}_s)\leq \frac{\lambda_t K^2}{8m_s}+\mathcal{L}(P,\mathcal{D}_s)\quad ; \quad
\mathbb{E}_{S_s,S_t\sim \mathcal{D}_s}\mathcal{L}( \hat{Q}^{\lambda_t}_{s:t},\mathcal{D}_t)\leq \frac{\lambda_t K^2}{8m_t}+\mathcal{L}(P,\mathcal{D}_t). 
\end{align*}
\end{corollary}


This Corollary is a specific case of the more general Theorem \ref{thm:forgetting-extended} that applies for any finite number of tasks. Like the more general Theorem, the loss covariance term depends on the expressive power of the hypothesis class $\mathcal{H}$ as well as the similarity between the tasks.



\subsection{Oracle bounds for multiple tasks}

As in previous sections, we would like to extend our results to more than two tasks, as well as attempt to derive bounds that are easier to quantify. In this subsection, we will derive an asymptotic oracle upper bound on the backward transfer loss for the empirical Gibbs posterior.

Suppose we are given a set of tasks $[T]\triangleq \{1,2,\ldots, T\}$, appearing sequentially. For simplicity we will assume that each task has identical sample of size $m$.
We would like to minimize the average backward transfer after all tasks, $BWT(Q_{1:T},[T])$. In general, this definition does not assume anything about the construction of the posterior, other than the fact that we have no access to samples from future (unseen) tasks while learning each specific posterior $Q_{1:t}$. 

In order to obtain quantifiable asymptotic bounds on the backward transfer, we make several simplifying assumptions on our hypothesis space and loss function:
\begin{assumption} \label{assume:laplace-things}
    The following five claims hold: 
    (1) The hypothesis space $\mathcal{H}$ is a compact bounded subset of $\mathbb{R}^d$.
    (2) $\ell(h,z)$ is twice continuously differentiable w.r.t.\! $h$.
    (3) For all $i\in[T-1]$, the function $\sum_{t=1,t\neq i}^{T}\mathcal{L}(h,\mathcal{D}_t)$ has a unique global minimum marked $h^*_{1:T/i}$.
    (4) For all $i\in[T-1]$, $P(h^*_{1:T/i})>0$.
    (5) For $m\rightarrow\infty$ \footnote{This holds if data within each task is sampled i.i.d.} $$\mathbb{E}_{S_1,\ldots,S_T}\mathrm{cov}_P\left (e^{\lambda_T\sum_{j=1}^{T}\mathcal{L}(h,\mathcal{D}_j)},e^{-\lambda_T\sum_{j=1}^{T}\hat{\mathcal{L}}(h,S_j)}\right )\leq 0.$$
%
\end{assumption}

Starting from \eqref{eq:oracle-base}, we show in appendix \ref{append:proofs} that if all of the priors are empirical Gibbs distributions, namely  
$\forall i\in\{2,\ldots,T\}, ~~\hat{Q}^{\lambda_i}_{1:i}(h)\propto \hat{Q}^{\lambda_{i-1}}_{1:i-1}(h)e^{-\lambda_i\hat{\mathcal{L}}(h,S_i)}$,  
where $\hat{Q}^{\lambda_1}_{1:1}(h)\propto P(h)e^{-\lambda_1\hat{\mathcal{L}}(h,S_1)}$, we have the following asymptotic oracle bound.

\begin{theorem} \label{thm:forgetting-asymptotic}
For any $\lambda_T>0$, assuming all $Q_{1:j}$ are empirical Gibbs posteriors, and that Assumption \ref{assume:laplace-things} holds, the following holds a.s. in expectation over $S_{j}\sim \mathcal{D}_j$, for all $j\in[T]$
\begin{align} \label{eq:forgetting-asymptotic}
\begin{split}
\lim_{m,T\rightarrow \infty}\frac{1}{T-1}\mathbb{E}_{S_1,\ldots,S_T}\sum_{i=1}^{T-1}\mathcal{L}(\hat{Q}^{\lambda_T}_{1:T}, \mathcal{D}_i) &\leq \lim_{m,T\rightarrow \infty}\frac{1}{T-1}\sum_{i=1}^{T-1}\mathcal{L}(h^*_{1:T/i}, \mathcal{D}_i).
\end{split}
\end{align}
\end{theorem}

The proof mainly uses Laplace's method \citep{shun1995laplace}. This result implies that (asymptotically) forgetting for a given task can be low if hypotheses that work well for other tasks can be applied to it with low loss. This may be the case due to task similarity or due to the expressive power of $\mathcal{H}$. We note that the assumption of a unique  global minimum may be relaxed to a finite set of global minima by replacing the loss in the r.h.s. of \eqref{eq:forgetting-asymptotic} with the highest loss among optimal hypotheses for other tasks, i.e. $\max_{h\in\mathcal{H}:h\in\arg\min_h \sum_{t=1,t\neq i}^{T}\mathcal{L}(h,\mathcal{D}_t)}\mathcal{L}(h, \mathcal{D}_i).$ 

\subsection{Oracle bounds for related tasks}

While the bound in Theorem \ref{thm:forgetting-asymptotic} is applicable for any task distribution, tighter bounds can be derived if tasks are sufficiently similar. To that end, we define the empirical task covariance
$$\mathrm{cov}_{P}(i, [T])\triangleq\mathrm{cov}_{P}(e^{-\lambda_T\hat{\mathcal{L}}(h,S_i)}, e^{-\sum_{j=1,j\neq i}^{T}\lambda_j\hat{\mathcal{L}}(h,S_j)}).$$ 

Unlike the two-task version of the loss covariance used in Corollary \ref{thm:cov-2task}, this definition compares the loss on task $i$ to the total loss for all other tasks. Intuitively, this covariance is positive if hypotheses that perform well for task $i$ tend to also work well on average for all other tasks (and vice versa for hypotheses with high loss). Since the prior $P$ is data-free w.r.t. all tasks, this is a measure of both task similarity and the expressive power of hypotheses in $\mathcal{H}$.

Using Laplace's method as well as this covariance term yields the following improved asymptotic bound:

\begin{theorem} \label{thm:forgetting-asymptotic-cov}
Assuming all $Q_{1:j}$ are empirical Gibbs posteriors with $\lambda_i=\sqrt{m}$, $P$ is a uniform prior over $\mathcal{H}$ , and that Assumption \ref{assume:laplace-things} holds with a condition on $h^*_i=\argmin_h\mathcal{L}(h,\mathcal{D}_i)$ instead of the condition on all other other tasks, 
if $\ \forall i\in[T-1],\  \mathrm{cov}_{P}(i, [T])\geq 0$,
the following holds a.s. in expectation over $S_{j}\sim \mathcal{D}_j$, $\forall j\in[T]$
\begin{align} \label{eq:forgetting-asymptotic-cov}
\begin{split}
\lim_{m\rightarrow \infty}\frac{1}{T-1}\mathbb{E}_{S_1,\ldots,S_T}\sum_{i=1}^{T-1}\mathcal{L}(\hat{Q}^{\sqrt{m}}_{1:T}, \mathcal{D}_i) &\leq \frac{1}{T-1}\sum_{i=1}^{T-1}\mathcal{L}(h^*_{i}, \mathcal{D}_i)+\lim_{m \rightarrow\infty}\frac{\mathbb{E}_{S}\sum_{i=1}^{T-1}\log \det \hat{\mathcal{L}}_i''(h^*_{i})}{2\sqrt{m}(T-1)}.
\end{split}
\end{align}
\end{theorem}

The proof is deferred to Appendix \ref{append:proofs}. 
Assuming that the determinant of the Hessian matrix is of order less than $e^{\sqrt{m}}$, which is likely unless $h^*_i$ is a very sharp global minimum, this result shows that the task covariance condition is sufficient to ensure that no forgetting occurs.
We note that since the task covariance is defined w.r.t.~$h\sim P$, we can expect the expressive power of $\mathcal{H}$ to impact whether or not the assumption $\mathrm{cov}_{P}(i, [T])\geq 0$ holds. This result differs from Theorem \ref{thm:forgetting-asymptotic} as the r.h.s. contains the optimal hypotheses for each task $h^*_1,
\ldots,h^*_T$ as opposed to hypotheses that are optimal for all other tasks. Since $h^*_i=\arg\min_h\mathcal{L}(h,\mathcal{D}_i)$, Theorem \ref{thm:forgetting-asymptotic-cov} is, by definition, optimally tight for $m\rightarrow \infty$ assuming that the determinant is low.

We additionally show in Appendix \ref{append:proofs} that the following non-asymptotic oracle bound holds.
\begin{theorem} $\mathrm{(Forgetting, Generalization)}$ 
\label{thm:forgetting-extended}
For any $\lambda_T>0$, assuming all $Q_{1:j}$ are empirical Gibbs posteriors, and that
 $\mathrm{cov}_{P}(i, [T])\geq 0$,
for any sample of training sets $S_{j}\sim \mathcal{D}_j$, then $\forall i\in[T]$
\begin{align} \label{eq:forgetting-extended}
\begin{split}
\mathbb{E}_{S_i}\mathcal{L}(\hat{Q}^{\lambda_T}_{1:T}, \mathcal{D}_i) &\leq \frac{\lambda_T K^2}{8m_i}+\mathcal{L}(P,\mathcal{D}_i).
\end{split}
\end{align}
\end{theorem}
Theorem \ref{thm:forgetting-extended} is somewhat surprising, as it implies a sufficient condition for learning without forgetting that, for each individual task, does not become worse with the number of tasks and has no direct dependence on task order or on the length of time a task has not been seen. While the r.h.s.~in \eqref{eq:forgetting-extended} contains a constant term $\mathcal{L}(P,\mathcal{D}_i)$, this term does not depend on the number of total tasks $T$. 

Although the resulting upper bound is potentially quite loose, it can be useful if we assume that the prior $P$ is some pre-trained model that already achieves low loss on all tasks. In this case, the Gibbs posterior allows us to iteratively fine-tune the pre-trained model while guaranteeing that the backward transfer loss is no worse than the performance of the original pre-trained model.



\section{Empirical validation} \label{sec:empirical}

In this section we demonstrate the utility of our bounds on both synthetic and real-world data sets. We study three classes of task environments, with very different types of shifts between tasks, and show, that for sufficiently many tasks, the bounds are tight. 
Since Corollary \ref{thm:forget-sumT-final} applies uniformly over all posteriors, it can be optimized at each time step w.r.t.~the current posterior.
Specifically, we select \begin{equation}    \label{eq:pb-update}
\hat{Q}_{1:i}=\argmin_{Q_{1:i}}\left \{\hat{\mathcal{L}}(Q_{1:i},S_i)+\frac{1}{\lambda_i}D_{\mathrm{KL}}(Q_{1:i}||Q_{1:i-1})\right \}.\end{equation}
This optimization process yields a theoretically motivated stochastic optimization algorithm, detailed in Algorithm \ref{alg:empirical-forgetting}.
\begin{algorithm}[H]
	\caption{PAC-Bayes Continual Learning}
	\label{alg:empirical-forgetting}
	\small
	\begin{algorithmic}[1]
		\STATE {\bfseries Continual-learn} ($S_1,\ldots, S_T$, $P$)
		\STATE Choose algorithmic parameters $\lambda_1,\ldots,\lambda_T$
		\STATE Let $\hat{Q}_{1:0}(h) \triangleq P(h)$
		\FOR {each task $t$ from $1$ to $T$} 
            \STATE Update $\hat{Q}_{1:t}$ via \eqref{eq:pb-update}, with $Q_{1:i-1}=\hat{Q}_{1:t-1}$
		\ENDFOR
		\STATE {\bfseries Return} $\hat{Q}_{1:T}$
	\end{algorithmic}
\end{algorithm}
This algorithm can be approximated in practice via variational inference (VI) using multivariate Gaussian prior and posterior distributions on model parameters.  The main benefit of doing so is to simplify the estimation of the KL-divergence. Other distributions are possible, but Gaussian distributions are common in VI \citep{hoffman2013stochastic}. The posterior can be used to estimate the forgetting and its upper bound by using a pre-defined test set as a proxy for the data distribution $\mathcal{D}_i$. Calculating the upper bound \eqref{eq:final-forget-bound} would add an additional calculation of $O(t\cdot (m+d))$ to each task, where $m=|S_t|$ and $d$ is the number of model parameters. In total, calculating the bound at each task will add $O(T^2\cdot (m+d))$. Calculating the bound does not require additional memory or samples (samples from previous test tasks are also required for calculating test forgetting).

We also examine our bounds for a deterministic continual learning algorithm, namely Elastic Weight Consolidation (EWC) \citep{kirkpatrick2017overcoming}. Since EWC outputs deterministic parameters for each task, a posterior distribution is constructed by adding Multivariate Gaussian noise, i.e. $\hat{Q}_{1:i}=\mathcal{N}(w_i,\sigma^2 I_d)$, where $w_i\in \mathcal{R}^d$ is the weight vector given as output after task $\mathcal{D}_i$. By doing so, the KL-divergence can be easily calculated as $$KL(\hat{Q}_{1:i}||\hat{Q}_{1:i-1})=\frac{1}{2\sigma^2}||w_i-w_{i-1}||^2_2.$$

\paragraph{Experimental settings}
In order to measure the bound tightness for VI and EWC, we must define the relevant metrics we wish to measure. These are the backwards transfer $BWT(\hat{Q}_{1:T},[T])$ (measured via a separate test set per task) and the forgetting $F(\hat{Q}_{1:T},[T])$, as well as the average forward transfer $\frac{1}{T}\sum_{t=1}^{T}\mathcal{L}(\hat{Q}_{1:t}, \mathcal{D}_t)$.

In the following experiments, we use fully connected neural networks with a separate linear head per task to facilitate measuring the forgetting without re-training linear heads. We use Adam \citep{KingmaB14} for optimization. The full list of hyper-parameters is listed in Appendix \ref{append:hyperparam}. 
%
\paragraph{10d Gaussian data} 
We begin by examining several setups of binary classification tasks in $\mathbb{R}^{10}$. All tasks draw samples from a $10$-dimensional Gaussian distribution $p(x) = \mathcal{N}(x;0,I_{10})$ and $y=\mathrm{sgn}(a^\top x)$. To simplify, only the first two features were used to determine $y$, meaning that there is a $2d$ linear separator embedded in $\mathbb{R}^{10}$.
We consider the following settings:
(1) Similar tasks: linear separators for all tasks are within $10^\circ$ of some reference angle.
(2) Gradual shift: tasks change angle gradually in a set direction, with each consecutive task being within $10^\circ$ of the previous one.
(3) Orthogonal shift: Tasks originate from two distinct, orthogonal directions, with the first half of tasks being from the first angle and the latter half corresponding to the second angle.  

We used a total of $T=100$ tasks in this domain in order to observe general trends.
We note that these problems differ significantly in task order and overall behavior, and are aimed at providing a diverse set of challenges. In addition, we should not expect a single algorithm to perform optimally across all settings, as known empirical methods based on regularization would likely perform poorly on orthogonal tasks whereas replay-based and architecture-based methods would have issues with gradual shift. We expect that in real-world settings, prior knowledge can be used to suggest specific algorithms.
In particular, the settings of Gradual shift and Orthogonal shifts are constructed such that forgetting may be desirable. A recent paper by \citet{kumar2023continual} discusses continual learning as a computationally constrained optimization problem and argues that forgetting non-recurring information is not ``catastrophic", especially given changing environments. In both the gradual and orthogonal shift settings, not all previous knowledge is beneficial.

\begin{table*}[t]
\caption{Metrics for $10d$ continual tasks at $T=100$.} 
\label{2d-full-table}
\vskip 0.1in
\begin{center}
\begin{small}
\begin{sc}
\begin{tabular}{lccccc}
Method  &BWT &BWT bound &Forgetting &Forgetting bound \eqref{eq:final-forget-bound} &Test loss \\
\midrule
 Similar tasks   \\
\midrule
EWC    & $2.0 \pm 0.1$ & $3.5 \pm 0.1$ & $0.4 \pm 0.0$ & $1.8 \pm 0.1$ & $1.8 \pm 0.4$  \\
VI & $\textbf{0.8}\pm 0.2$ & $\textbf{2.2}\pm 0.2$ & $\textbf{-1.9} \pm 0.1$ & $\textbf{-0.5} \pm 0.1$ & $\textbf{1.0}\pm 0.5$ \\
\midrule
 Gradual shift   \\
 \midrule
EWC    & $1.9 \pm 0.1$ & $3.3 \pm 0.1$ & $0.3 \pm 0.1$ & $1.7 \pm 0.1$ & $2.1 \pm 0.2$  \\
VI & $\textbf{0.2}\pm 0.0$ & $\textbf{1.5}\pm 0.0$ & $\textbf{-2.2} \pm 0.0$ & $\textbf{-0.8} \pm 0.0$ & $\textbf{0.2}\pm 0.1$ \\
\midrule
 Orthogonal shift   \\
 \midrule
EWC    & $4.1 \pm 0.8$ & $5.5 \pm 0.9$ & $0.8 \pm 0.2$ & $2.2 \pm 0.2$ & $4.6 \pm 0.9$  \\
VI & $\textbf{3.3}\pm 0.8$ & $\textbf{4.7} \pm 0.8$ & $\textbf{-0.0} \pm 0.7$ & $\textbf{1.4} \pm 0.7$ & $\textbf{0.5}\pm 0.1$ \\
\end{tabular}
\end{sc}
\end{small}
\end{center}
\vskip -0.1in
\end{table*}

\begin{figure*}[t]
\centering
\subfigure[Gradual shift\label{10d-sim}]{\includegraphics[width=0.33\textwidth]{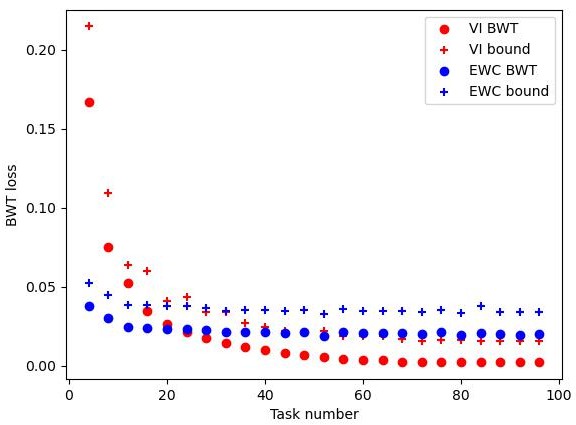}} 
\subfigure[Orthogonal shift at $t=50$\label{10d-shift}]{\includegraphics[width=0.33\textwidth]{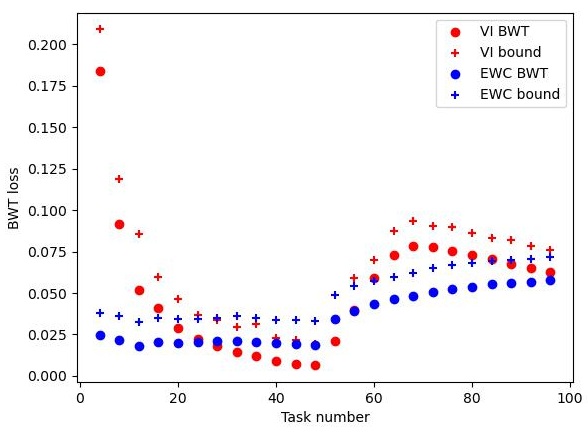}} 
\caption{Forgetting and upper bound \eqref{eq:final-forget-bound} over time ($t=4$ to $t=96$) for $10d$ data. }
\label{10d-forget}
\end{figure*}

Table \ref{2d-full-table} describes the forward and backward transfer as well as forgetting for the $10d$ tasks. 
All test metrics are in loss percentage over the relevant models and test sets averaged over $5$ seeds. Backward transfer is the final loss percentage over all test sets. Standard error is reported after the $\pm$ sign. Best result for each setting is in bold. We can see that for most settings, VI based on Algorithm \ref{alg:empirical-forgetting} tends to provide good forward transfer, as well as backward transfer and forgetting. For the orthogonal shift setting, we can see that both backward transfer and forgetting degrade. Figure \ref{10d-forget} compares the forgetting for the gradual shift setting compared to the more abrupt orthogonal shift at $T=50$. We can clearly see that the backward transfer increases abruptly at the point of task shift, though the upper bound remains roughly the same in terms of tightness as the task discrepancy term is negative and corrects for the higher KL-divergence at the task shift.

The forgetting for VI is negative and seems to follow a trend of increasing gradually towards zero forgetting. For the orthogonal shift setting, the forgetting becomes positive and decreases towards zero.
For EWC, the forgetting is positive and decreases gradually towards zero.
We observe that as the shared model improves, forgetting decreases, likely due to the learned shared representation improving overall across several tasks, and we see the forgetting trend towards zero. 
An interesting observation is that the backward transfer (and forgetting) for the orthogonal shift setting is much higher. In this setting, forgetting may be a desired phenomenon, as the data distribution shifts in the middle of the learning process. We see that indeed this shift leads to higher forgetting. 

In appendix \ref{append:hyperparam} we also consider the discounted backward transfer and forgetting, using the formulae
\begin{align*}\begin{split}
&BWT^{\gamma}(Q_{1:T},[T])\triangleq \sum_{t=1}^{T-1}\gamma^{T-1-t}\mathcal{L}(Q_{1:T},\mathcal{D}_t),\quad
    F^{\gamma}(Q_{1:T},[T])\triangleq \sum_{t=1}^{T-1}\gamma^{T-1-t}(\mathcal{L}(Q_{1:T},\mathcal{D}_t)-\mathcal{L}(Q_{1:t}), \mathcal{D}_t).
\end{split}\end{align*}
These discounted forms allow us to re-weigh the forgetting based on recency, as achieving low forgetting on recent tasks may be more important in some settings than low overall forgetting. For $\gamma=0.95$, the discounted backward transfer for VI is lower, and the forgetting is higher. This is likely due to the shared model improving over time. For EWC, the backward transfer increases until a low steady state, and forgetting decreases gradually.
\begin{figure*}[ht!]
\centering
\subfigure{\includegraphics[width=0.25\textwidth]{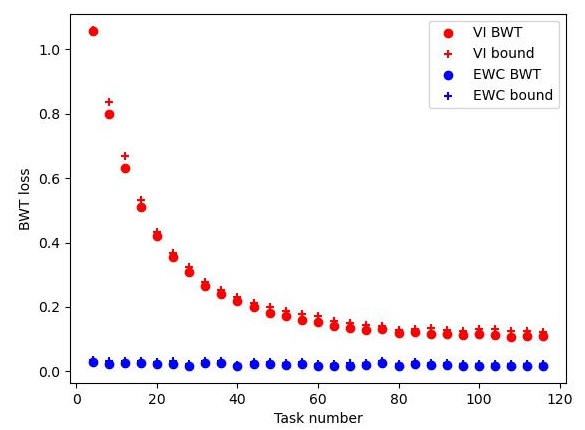}}
\subfigure{\includegraphics[width=0.25\textwidth]{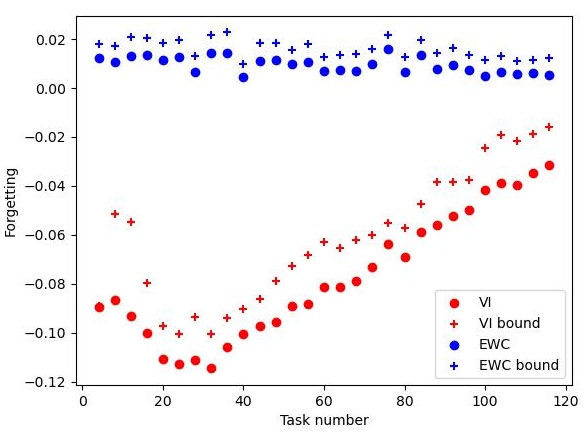}}
\subfigure{\includegraphics[width=0.25\textwidth]{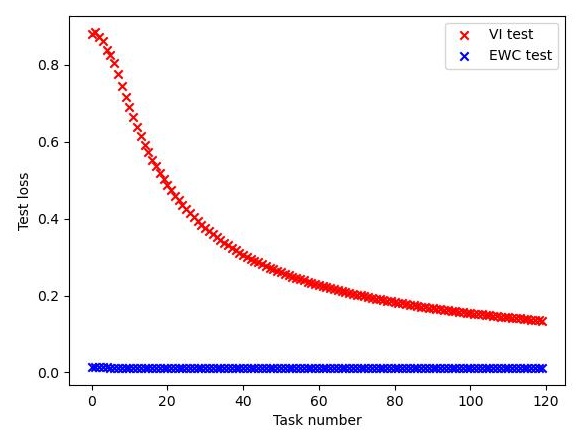}}
\\%
\subfigure{\includegraphics[width=0.25\textwidth]{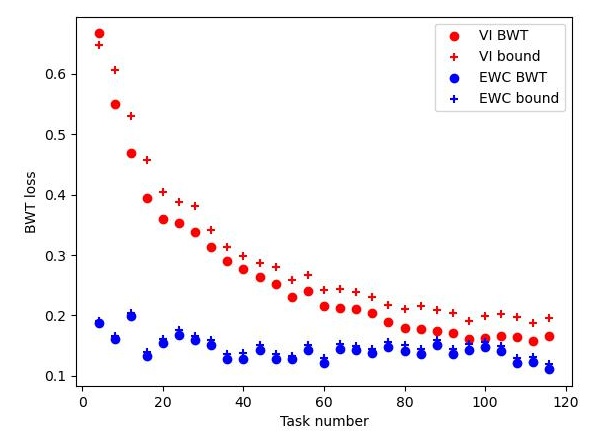}} 
\subfigure{\includegraphics[width=0.25\textwidth]{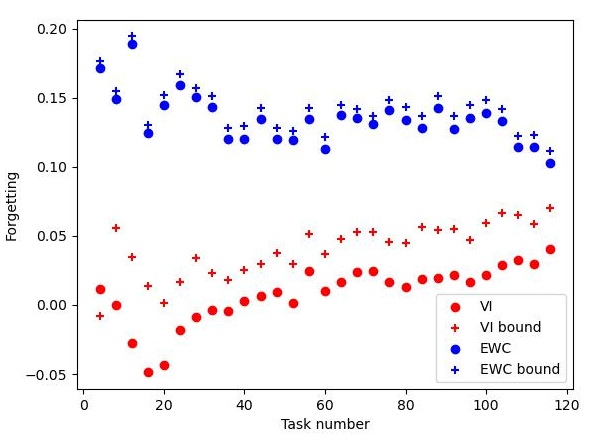}} 
\subfigure{\includegraphics[width=0.25\textwidth]{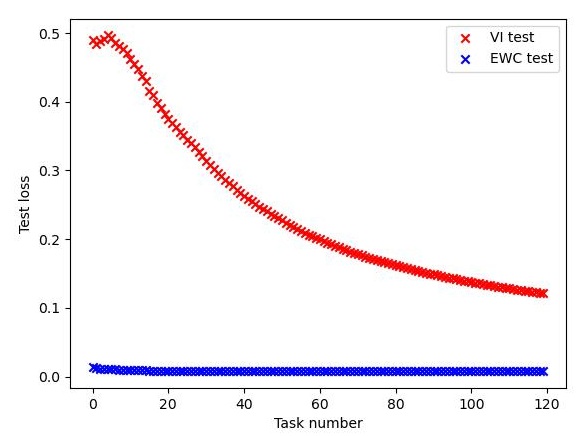}}
\\%
\addtocounter{subfigure}{-6}
\subfigure[BWT\label{bwt}]{\includegraphics[width=0.25\textwidth]{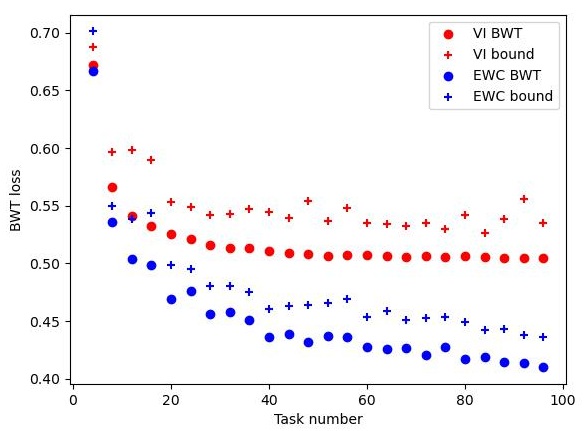}}
\subfigure[Forgetting\label{forgetting}]{\includegraphics[width=0.25\textwidth]{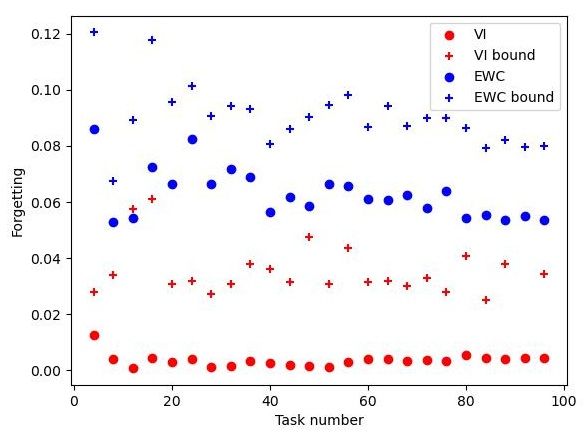}}
\subfigure[Test loss\label{test}]{\includegraphics[width=0.25\textwidth]{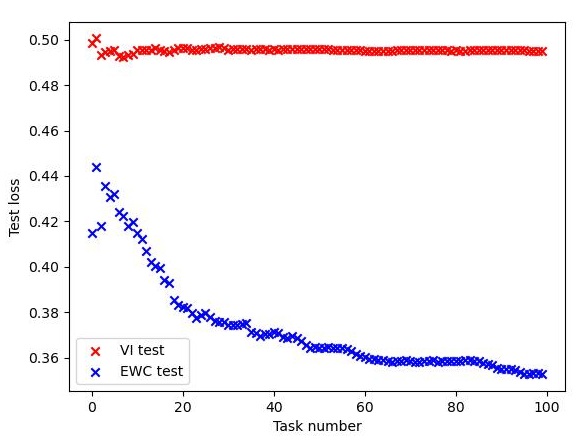}}
\caption{Metrics over time for permuted-MNIST (top row), split-MNIST (middle row) and split-CIFAR10 (bottom row). (a) backward transfer \& upper bound (b) forgetting \& upper bound \eqref{eq:final-forget-bound} (c) average (forward) test error over time.} 
\label{mnist-all}
\end{figure*}

\paragraph{Vision tasks} 
Next we examine our bounds for our approximation of Algorithm \ref{alg:empirical-forgetting} and EWC on more realistic problem domains, namely permuted-MNIST \citep{goodfellow2015empirical}, a domain-incremental continual learning problem \citep{de2021continual} where a random permutation of the pixels is applied on each image for each task, split-MNIST \citep{zenke2017continual}, a sequential set of binary classification tasks constructed from the MNIST dataset, and split-CIFAR10 \citep{zenke2017continual}, a sequential set of binary classification tasks constructed from the CIFAR-10 dataset. We used convolutional neural networks for these visual domains, with $T=120$ tasks in total for MNIST domains and $T=100$ for split-CIFAR10. For these more challenging tasks, our practical approximation of Algorithm \ref{alg:empirical-forgetting}, that works well for synthetic data, does not perform as well as EWC in terms of forward and backward transfer.

Figure \ref{bwt} shows the gradual change in backward transfer.
We can see that after an initial warm-up period, all settings have gradually improving backward transfer as well as tight upper bounds \eqref{eq:forget-base-T}, especially for EWC.
Figure \ref{test} shows that for both MNIST settings, the test error for VI decreases with additional tasks. This suggests that the shared representation has sufficient network capacity to learn a useful representation for all MNIST vision tasks. For EWC the loss also decreases, but rapidly reaches near zero. For the split-CIFAR10 setting, we see that VI fails to learn useful hypotheses, and EWC performs reasonably well.

Figure \ref{forgetting} shows the forgetting and its upper bound \eqref{eq:final-forget-bound} over time. For both MNIST problems, we see similar behavior to that of the $10d$ Gaussian data, where the forgetting trends upwards toward zero. For split-CIFAR10 we see the forgetting for EWC remain at a mostly constant low level as the backward transfer and test losses improve slowly. We see that the bounds for split-CIFAR10 are reasonably tight. The poor performance of VI compared to EWC on the split-CIFAR10 can be attributed to a combination of low computational budget and modeling choices. VI methods often require additional computation epochs compared to deterministic neural networks, and computational parity with EWC was prioritized. It may be the case that the approximation via an isotropic Gaussian prior served to limit the search space for optimization.
The full tables of the results appear in Appendix \ref{append:hyperparam}.

\section{Limitations}
Our upper bounds assume that tasks are sampled independently of the learning process. This assumption is necessary for our proofs and while reasonable for supervised learning (outside of curriculum learning), would not hold under adversarial task constructions. 
We also assume that the loss function $\ell(h,z)$ is task-independent. This assumption can likely be relaxed via separate loss functions per task without invalidating the proof.
Our bounds assume that the the continual learner is exemplar-free (data from previous tasks is not directly available). It may be the case that additional data availability may be used to construct tighter bounds.
Additional assumptions are listed throughout the text, with many of our results assuming that the loss is bounded (Assumption \ref{assume:bounded-loss}). Most of these results can be extended to low-tailed losses, such as sub-Gaussian losses, but would not hold for heavy-tailed losses.

Our empirical results are mainly focused on measuring the tightness of our upper bounds. As such, we focused on relatively small neural networks and reasonably simple supervised problems in order to validate our bounds in practice. Since the KL-divergence is part of the upper bound \eqref{eq:final-forget-bound}, it may be the case that this bound will become looser for larger models, as the KL-divergence tends to increase with the number of parameters. 
Compared to standard PAC-Bayes bounds this issue is somewhat alleviated due to only requiring the KL-divergence for the final model, meaning that this term can be mitigated somewhat by ensuring that model changes lessen over time.
Outside of our synthetic dataset, all of our experiments used vision datasets, and we only examined our bounds on our simple stochastic algorithm and EWC, and therefore our experiments should be considered preliminary in scope.

\section{Conclusions} 
In this work, we derived several upper bounds on the test forgetting (via backward transfer) for both general model classes and for the Gibbs posterior, based on the change of measure approach.
These upper bounds are data-dependent, offering tighter bounds if improved prior models for the task are available, or if the loss landscapes for tasks are similar. In particular, we focused on oracle bounds for Gibbs posteriors that offered tight bounds on backward transfer if task losses are highly positively correlated, thus making the knowledge accumulation process highly effective for all tasks.

Based on our theoretical bounds, we constructed an algorithm for continual learning with empirically low forgetting and tight upper bounds. We examined our bounds for both this algorithm and EWC on several simple task constructions as well as slightly more complex vision tasks. In our experiments, we noted a relation between forward and backward transfer, especially for mostly static settings such as the domain-incremental continual learning problem. As noted by several previous theoretical works and practical examinations (see Introduction), task order and similarity can greatly influence both forgetting and generalization. While this empirical demonstration is not the main focus of our paper, it suggests that weighting schemes based on notions of loss agreement merit further exploration for the domain incremental setting.

\section*{Acknowledgements}
This work was supported by the Israel Science Foundation grant number $1693/22$.

\clearpage

\bibliography{library}
\bibliographystyle{collas2025_conference}

\newpage
\appendix
\onecolumn

\section{Appendix - notation glossary} \label{append:notation}

\begin{tabular}{p{1.25in}p{3.25in}}
$\mathcal{X}$ & Feature set \\
$x\in \mathcal{X}$ & Data sample \\
$\mathcal{Y}$  & Label set \\
$y\in\mathcal{Y}$  & Data label \\
$T$ & Number of tasks\\
$\mathcal{D}$ & Task distribution over pairs $(x,y)$\\
m & Size of training set\\
$S\sim\mathcal{D}^m$ & Training set of size $m$\\
$\mathcal{H}$ & Hypothesis set\\
$h\in\mathcal{H}, h:\mathcal{X}\rightarrow\mathcal{Y}$ & Hypothesis\\
$\ell: \mathcal{H}\times(\mathcal{X},\mathcal{Y})\rightarrow\mathbb{R}$ & Loss function\\
$\mathcal{L}:\mathcal{H}\times\mathcal{D}\rightarrow\mathbb{R}$ & Expected loss\\
$\hat{\mathcal{L}}:\mathcal{H}\times(\mathcal{X},\mathcal{Y})^m\rightarrow\mathbb{R}$ & Empirical loss\\
$P$ & Data-free prior distribution over hypotheses \\
$Q$ & Data-dependent posterior distribution over hypotheses \\
$Q_{1:i}$ & Data-dependent posterior distribution depending on $S_1,\ldots,S_i$ \\
$\lambda>0$ & KL-divergence weight (Gibbs posterior temperature) \\
$K$ & Maximal loss value \\
$\delta\in(0,1)$ & A probability \\
$d$ & Dimension of hypothesis space under Assumption \ref{assume:laplace-things} \\
$\sigma^2$ & Variance (Gaussian distribution) \\
\end{tabular}

\section{Appendix - proofs} \label{append:proofs}


\begin{lemma} \label{lemma:concentration} \cite{shui2020beyond} Let $\pi$ and $\rho$ be two distributions on a common space $\mathcal{Z}$ such that $\rho$ is absolutely continuous w.r.t.~$\pi$. For any $\lambda_t\in \mathbb{R}$ and any measurable function $f:\mathcal{Z}\rightarrow \mathbb{R}$ such that $\mathbb{E}_{z\sim \pi}\left [e^{\lambda_t(f(z)-\mathbb{E}_\pi f(z))} \right ]<\infty$, we have
	\begin{equation}
 \begin{split}
	\lambda_t&\left ( \mathbb{E}_{z\sim \rho}\left [f(z) \right ]-\mathbb{E}_{z\sim \pi}\left [f(z) \right ] \right ) \leq 
 D_{\mathrm{KL}}(\rho||\pi)+ \log\mathbb{E}_{z\sim \pi}\left [e^{\lambda_t(f(z)-\mathbb{E}_\pi f(z))} \right ],
 \end{split}
	\end{equation}	
	where $D_{\mathrm{KL}}$ is the KL-divergence and equality is achieved for $f(z)=\mathbb{E}_{z\sim\pi} f(z)+\frac{1}{\lambda_t}\log(\frac{d\rho}{d\pi})$.
\end{lemma}

\begin{theorem} Restatement of Theorem \ref{thm:forget-base2}:
     For any fixed $S_s,S_t,Q_s,Q_{s:t}$, for any $
    \lambda_t>0$,
    \begin{align} 
\begin{split}
\mathcal{L}(Q_{s:t}, \mathcal{D}_s) &\leq \hat{\mathcal{L}}(Q_{s:t}, S_t) + \frac{1}{\lambda_t} D_{\mathrm{KL}}(Q_{s:t}||Q_{s})+\frac{1}{\lambda_t}\log\mathbb{E}_{h\sim Q_{s}}\left [e^{\lambda_t(\mathcal{L}(h,\mathcal{D}_s)-\hat{\mathcal{L}}(h,S_t))} \right ]
\end{split}
\end{align}
\end{theorem}

\begin{proof}
    Starting from Lemma \ref{lemma:concentration}, we can choose $f(z)=\lambda_t(\mathcal{L}(z,\mathcal{D}_s)-\hat{\mathcal{L}}(z,S_t))$, giving us

\begin{align*}
&\lambda_t\mathbb{E}_{h\sim Q_{s:t}}\left [\mathcal{L}(h,\mathcal{D}_s)-\hat{\mathcal{L}}(h,S_t) \right ] - \lambda_t\mathbb{E}_{h\sim Q_{s}}\left [\mathcal{L}(h,\mathcal{D}_s)-\hat{\mathcal{L}}(h,S_t) \right ] \\
&~~\leq D_{\mathrm{KL}}(Q_{s:t}||Q_{s})+\log\mathbb{E}_{h\sim Q_{s}}\left [e^{\lambda_t(\mathcal{L}(h,\mathcal{D}_s)-\hat{\mathcal{L}}(h,S_t))}e^{-\lambda_t(\mathcal{L}(Q_s,\mathcal{D}_s)-\hat{\mathcal{L}}(Q_s,S_t))} \right ]
\end{align*}

Extracting terms that do not depend on $h$ from the expectation, we get

\begin{align} \label{eq:forget-base}
\begin{split}
F(Q_{s:t},\mathcal{D}_s) &\leq \hat{\mathcal{L}}(Q_{s:t}, S_t) - \mathcal{L}(Q_{s}, D_s) + \frac{1}{\lambda_t} D_{\mathrm{KL}}(Q_{s:t}||Q_{s})\\
&+\frac{1}{\lambda_t}\log\mathbb{E}_{h\sim Q_{s}}\left [e^{\lambda_t(\mathcal{L}(h,\mathcal{D}_s)-\hat{\mathcal{L}}(h,S_t))} \right ]
\end{split}
\end{align}

\end{proof}

\begin{lemma} \label{lemma:hoeffding-concentration} (See also Lemma $1.1$ in \citet{alquier2021user})
	Let $l:Z\times H\rightarrow[0,K]$ be a measurable function. Let $\pi\in\mathcal{M}(H)$ be a distribution over $H$ that is independent w.r.t. $Z$. Let $S\in Z^m$ be an i.i.d.~sample. With probability at least $1-\delta$ over the choice of $S$,
	$$\log \mathbb{E}_{h\sim \pi}\left [e^{t(\frac{1}{m}\sum_i l(z_i,h)-\mathbb{E}_{z}l(z,h))}\right ]\leq \frac{t^2K^2}{8m}+\log{1/ \delta}$$
\end{lemma}

\begin{proof} 
	Using Markov's inequality, we know that 
	$$\textrm{Pr}\left (\mathbb{E}_{h\sim \pi}\left [e^{t(\frac{1}{m}\sum_i l(z_i,h)-\mathbb{E}_{z}l(z,h))}\right ]<\frac{1}{\delta}\mathbb{E}_{S\sim Z^m}\mathbb{E}_{h\sim \pi}\left [e^{t(\frac{1}{m}\sum_i l(z_i,h)-\mathbb{E}_{z}l(z,h))}\right ] \right ) \geq 1-\delta$$
	
	Applying Fubini's theorem (both distributions are independent), we can re-order the expectations
	
	$$\textrm{Pr}\left (\mathbb{E}_{h\sim \pi}\left [e^{t(\frac{1}{m}\sum_i l(z_i,h)-\mathbb{E}_{z}l(z,h))}\right ]<\frac{1}{\delta}\mathbb{E}_{h\sim \pi}\mathbb{E}_{S\sim Z^m}\left [e^{t(\frac{1}{m}\sum_i l(z_i,h)-\mathbb{E}_{z}l(z,h))}\right ] \right ) \geq 1-\delta$$
	
	Since $S$ is drawn i.i.d.~and $l$ is bounded, we can apply Hoeffding's lemma to each example, giving us
	
	$$\textrm{Pr}\left (\mathbb{E}_{h\sim \pi}\left [e^{t(\frac{1}{m}\sum_i l(z_i,h)-\mathbb{E}_{z}l(z,h))}\right ]<\frac{1}{\delta}\mathbb{E}_{h\sim \pi}\left [e^{\frac{t^2K^2}{8m}}\right ] \right ) \geq 1-\delta$$
	
	$$\textrm{Pr}\left (\log\mathbb{E}_{h\sim \pi}\left [e^{t(\frac{1}{m}\sum_i l(z_i,h)-\mathbb{E}_{z}l(z,h))}\right ]<\log\frac{1}{\delta}e^{\frac{t^2K^2}{8m}} \right ) \geq 1-\delta$$
	
	and so we have 
	
	$$\textrm{Pr}\left (\log\mathbb{E}_{h\sim \pi}\left [e^{t(\frac{1}{m}\sum_i l(z_i,h)-\mathbb{E}_{z}l(z,h))}\right ]<\log\frac{1}{\delta}+\frac{t^2K^2}{8m} \right ) \geq 1-\delta$$
	
\end{proof}


\begin{corollary} 
If $\ell\in [0,K]$, for any fixed $S_s,Q_s,\lambda_t>0$, the following applies uniformly for all posteriors $Q_{s:t}$ with probability at least $1-\delta$ over the choice of $S_t$,
\begin{align}
\begin{split}
\mathcal{L}(Q_{s:t}, \mathcal{D}_s) &\leq \hat{\mathcal{L}}(Q_{s:t}, S_t) + \frac{1}{\lambda_t} D_{\mathrm{KL}}(Q_{s:t}||Q_{s})\\
&+\frac{1}{2\lambda_t}\log \mathbb{E}_{h\sim Q_{s}}\left [e^{2\lambda_t(\mathcal{L}(h,\mathcal{D}_s)-\mathcal{L}(h,\mathcal{D}_t))}\right ]+\frac{\lambda_t K^2}{4m_t}+\frac{1}{2\lambda_t}\log(1/\delta)
\end{split}
\end{align}
\end{corollary}

\begin{proof}
    Starting from \eqref{eq:forget-base}, we note that
    $$\log\mathbb{E}_{h\sim Q_{s}}\left [e^{\lambda_t(\mathcal{L}(h,\mathcal{D}_s)-\hat{\mathcal{L}}(h,S_t))} \right ] = \log\mathbb{E}_{h\sim Q_{s}}\left [e^{\lambda_t(\mathcal{L}(h,\mathcal{D}_s)-\mathcal{L}(h,\mathcal{D}_t)+\mathcal{L}(h,\mathcal{D}_t)-\hat{\mathcal{L}}(h,S_t))} \right ].$$

    This gives us 
\begin{align*}
\log\mathbb{E}_{h\sim Q_{s}}\left [e^{\lambda_t(\mathcal{L}(h,\mathcal{D}_s)-\hat{\mathcal{L}}(h,S_t))} \right] &= \log\mathbb{E}_{h\sim Q_{s}}\left [e^{\lambda_t(\mathcal{L}(h,\mathcal{D}_s)-\mathcal{L}(h,\mathcal{D}_t))}e^{\lambda_t(\mathcal{L}(h,\mathcal{D}_t)-\hat{\mathcal{L}}(h,S_t))} \right ]\\
&\triangleq \log\mathbb{E}_{Q_{s}}\left [e^{\lambda_t\Delta\mathcal{L}(h,\mathcal{D}_s, \mathcal{D}_t)}e^{\lambda_t\Delta\hat{\mathcal{L}}(h,\mathcal{D}_t, S_t)} \right ].
\end{align*}

Similarly to Lemma \ref{lemma:hoeffding-concentration} we can apply Fubini's theorem, and since $\Delta\mathcal{L}(h, \mathcal{D}_s,\mathcal{D}_t)$ is independent of $S_t$ we can re-order the expectations and use Hoeffding's lemma on $e^{\lambda_t\Delta\mathcal{L}(h, \mathcal{D}_t,S_t)}$
if $\ell\in [0,K]$. More formally,

$$\mathbb{E}_{S_t}\mathbb{E}_{Q_{s}}\left [e^{\lambda_t\Delta\mathcal{L}(h,\mathcal{D}_s, \mathcal{D}_t)}e^{\lambda_t\Delta\hat{\mathcal{L}}(h,\mathcal{D}_t, S_t)} \right ]=\mathbb{E}_{Q_{s}}\mathbb{E}_{S_t}\left [e^{\lambda_t\Delta\mathcal{L}(h,\mathcal{D}_s, \mathcal{D}_t)}e^{\lambda_t\Delta\hat{\mathcal{L}}(h,\mathcal{D}_t, S_t)} \right ]$$
$$=\mathbb{E}_{Q_{s}}\left [\mathbb{E}_{S_t}\left [e^{\lambda_t\Delta\hat{\mathcal{L}}(h,\mathcal{D}_t, S_t)} \right ]e^{\lambda_t\Delta\mathcal{L}(h,\mathcal{D}_s, \mathcal{D}_t)}\right ]\leq \mathbb{E}_{Q_{s}}\left [e^{\lambda_t^2K^2/8m_t}e^{\lambda_t\Delta\mathcal{L}(h,\mathcal{D}_s, \mathcal{D}_t)}\right ]$$

Finally, we get with probability at least $1-\delta$ over the choice of $S_t$

$$\log\mathbb{E}_{h\sim Q_{s}}\left [e^{\lambda_t(\mathcal{L}(h,\mathcal{D}_s)-\hat{\mathcal{L}}(h,S_t))} \right ]\leq \log\mathbb{E}_{Q_{s}}\left [e^{\lambda_t\Delta\mathcal{L}(h,\mathcal{D}_s, \mathcal{D}_t)}\right ]+\frac{\lambda_t^2K^2}{8m_t}+\log(1/\delta)$$

\end{proof}



\begin{corollary} Restatement of Corollary \ref{thm:oracle-base}:
For any $Q_s, S_s, \lambda_t>0$, if $\ell\in [0,K]$, 
\begin{equation} 
\begin{split}
\mathbb{E}_{S_t\sim \mathcal{D}_t}&\mathcal{L}( \hat{Q}^{\lambda_t}_{s:t},\mathcal{D}_s)\leq\inf_{Q_{s:t}}\left \{ \mathcal{L}(Q_{s:t},\mathcal{D}_t) + \frac{1}{\lambda_t}D_{\mathrm{KL}}(Q_{s:t}||Q_{s}) \right \} \\
&+\frac{\lambda_t K^2}{8m_t}+\frac{1}{\lambda_t}\log\mathbb{E}_{h\sim Q_s}\left [e^{\lambda_t(\mathcal{L}(h,\mathcal{D}_s)-\mathcal{L}(h,\mathcal{D}_t))} \right ]
\end{split}
\end{equation}
\end{corollary}

\begin{proof}
Starting from Lemma \ref{lemma:concentration}, we know that 

$$\mathbb{E}_{z\sim \rho}\left [f(z) \right ]\leq \mathbb{E}_{z\sim \pi}\left [f(z) \right ]+ \frac{1}{\lambda_t}D_{\mathrm{KL}}(\rho||\pi)+ \frac{1}{\lambda_t}\log\mathbb{E}_{z\sim \pi}\left [e^{\lambda_t(f(z)-\mathbb{E}_\pi f(z))} \right ]$$

In particular, for $\hat{\rho}_\lambda(z)\propto \pi(z) e^{-\lambda_t f(z) }$, this is an equality (from \citeauthor{donsker1975large}'s [\citeyear{donsker1975large}] variational lemma).
From this, we know that

\begin{equation}
\mathbb{E}_{z\sim \hat{\rho}_\lambda}\left [f(z) \right ]= \mathbb{E}_{z\sim \pi}\left [f(z) \right ]+ \frac{1}{\lambda_t}D_{\mathrm{KL}}(\hat{\rho}_\lambda||\pi)+ \frac{1}{\lambda_t}\log\mathbb{E}_{z\sim \pi}\left [e^{\lambda_t(f(z)-\mathbb{E}_\pi f(z))} \right ]
\end{equation}

If we pick $f(z)=\mathcal{L}(z,\mathcal{D}_s)-\hat{\mathcal{L}}(z,S_t)$ as before, we get

\begin{equation*} 
\begin{split}
\mathbb{E}_{z\sim \hat{\rho}_\lambda}\left [\mathcal{L}(z,\mathcal{D}_s) \right ]&= \mathbb{E}_{z\sim \pi}\left [f(z) \right ]+\mathbb{E}_{z\sim \hat{\rho}_\lambda}\left [\hat{\mathcal{L}}(z,S_t) \right ]+ \frac{1}{\lambda_t}D_{\mathrm{KL}}(\hat{\rho}_\lambda||\pi)\\&+ \frac{1}{\lambda_t}\log\mathbb{E}_{z\sim \pi}\left [e^{\lambda_t(f(z)-\mathbb{E}_\pi f(z))} \right ]
\end{split}
\end{equation*}

And as such,
$$F( \hat{\rho}_\lambda,\mathcal{D}_s)\leq \inf_{\rho}\left \{ \hat{\mathcal{L}}(\rho,S_t) + \frac{1}{\lambda_t}D_{\mathrm{KL}}(\rho||\pi)  \right \}-\mathcal{L}(\pi,D_s)+\frac{1}{\lambda_t}\log\mathbb{E}_{z\sim \pi}\left [e^{\lambda_t(\mathcal{L}(z,\mathcal{D}_s)-\hat{\mathcal{L}}(z,S_t))} \right ],$$

or using our previous terminology with $\hat{Q}^{\lambda_t}_{s:t}(h)\propto Q_s(h)e^{-\lambda_t\hat{\mathcal{L}}(h,S_t)}$, 

$$\mathcal{L}( \hat{Q}^{\lambda_t}_{s:t},\mathcal{D}_s)\leq \inf_{Q_{s:t}}\left \{ \hat{\mathcal{L}}(Q_{s:t},S_t) + \frac{1}{\lambda_t}D_{\mathrm{KL}}(Q_{s:t}||Q_{s}) \right \}+\frac{1}{\lambda_t}\log\mathbb{E}_{h\sim Q_s}\left [e^{\lambda_t(\mathcal{L}(h,\mathcal{D}_s)-\hat{\mathcal{L}}(h,S_t))} \right ]$$

If we take an expectation on $S_t$, we get

\begin{equation*} 
\begin{split}
\mathbb{E}_{S_t\sim \mathcal{D}_t}\mathcal{L}( \hat{Q}^{\lambda_t}_{s:t},\mathcal{D}_s)&\leq \mathbb{E}_{S_t\sim \mathcal{D}_t}\inf_{Q_{s:t}}\left \{ \hat{\mathcal{L}}(Q_{s:t},S_t) + \frac{1}{\lambda_t}D_{\mathrm{KL}}(Q_{s:t}||Q_{s}) \right \}\\&+\frac{1}{\lambda_t}\mathbb{E}_{S_t\sim \mathcal{D}_t}\log\mathbb{E}_{h\sim Q_s}\left [e^{\lambda_t(\mathcal{L}(h,\mathcal{D}_s)-\hat{\mathcal{L}}(h,S_t))} \right ]
\end{split}
\end{equation*}

This gives us the following oracle inequality (in expectation):

\begin{equation*} 
\begin{split}
\mathbb{E}_{S_t\sim \mathcal{D}_t}\mathcal{L}( \hat{Q}^{\lambda_t}_{s:t},\mathcal{D}_s)&\leq \inf_{Q_{s:t}}\left \{ \mathcal{L}(Q_{s:t},\mathcal{D}_t) + \frac{1}{\lambda_t}D_{\mathrm{KL}}(Q_{s:t}||Q_{s}) \right \}\\&+\frac{1}{\lambda_t}\log\mathbb{E}_{h\sim Q_s}\mathbb{E}_{S_t\sim \mathcal{D}_t}\left [e^{\lambda_t(\mathcal{L}(h,\mathcal{D}_s)-\hat{\mathcal{L}}(h,S_t))} \right ]
\end{split}
\end{equation*}

We have

$$\frac{1}{\lambda_t}\log\mathbb{E}_{h\sim Q_s}\mathbb{E}_{S_t\sim \mathcal{D}_t}\left [e^{\lambda_t(\mathcal{L}(h,\mathcal{D}_s)-\hat{\mathcal{L}}(h,S_t))} \right ]$$
$$=\frac{1}{\lambda_t}\log\mathbb{E}_{h\sim Q_s}\mathbb{E}_{S_t\sim \mathcal{D}_t}\left [e^{\lambda_t(\mathcal{L}(h,\mathcal{D}_s)-\mathcal{L}(h,\mathcal{D}_t))}e^{\lambda_t(\mathcal{L}(h,\mathcal{D}_t)-\hat{\mathcal{L}}(h,S_t))} \right ]$$

$$=\frac{1}{\lambda_t}\log\mathbb{E}_{h\sim Q_s}\left [e^{\lambda_t(\mathcal{L}(h,\mathcal{D}_s)-\mathcal{L}(h,\mathcal{D}_t))}\mathbb{E}_{S_t\sim \mathcal{D}_t}e^{\lambda_t(\mathcal{L}(h,\mathcal{D}_t)-\hat{\mathcal{L}}(h,S_t))} \right ]$$

Using Hoeffding's lemma, we get

$$\frac{1}{\lambda_t}\log\mathbb{E}_{h\sim Q_s}\mathbb{E}_{S_t\sim \mathcal{D}_t}\left [e^{\lambda_t(\mathcal{L}(h,\mathcal{D}_s)-\hat{\mathcal{L}}(h,S_t))} \right ] \leq \frac{\lambda_t K^2}{8m_t}+\frac{1}{\lambda_t}\log\mathbb{E}_{h\sim Q_s}\left [e^{\lambda_t(\mathcal{L}(h,\mathcal{D}_s)-\mathcal{L}(h,\mathcal{D}_t))} \right ]$$

\end{proof}

\begin{theorem} Appendix-only Theorem: 
For any $\lambda_T>0$, assuming all $Q_{1:j}$ are empirical Gibbs posteriors, $\ell\in[0,K]$, 
the following applies a.s. in expectation over $S_{j}\sim \mathcal{D}_j$,
\begin{align} \label{eq:oracle-gibbs-ratio}
\begin{split}
\forall i\in[1,T-1]:
\mathbb{E}_{S_i}\mathcal{L}(\hat{Q}^{\lambda_T}_{1:T}, \mathcal{D}_i) &\leq \frac{\lambda_T K^2}{8m_i}+\frac{1}{\lambda_T}\mathbb{E}_{S_i}\log\frac{\mathbb{E}_{h\sim P}\left [e^{-\sum_{j=1,j\neq i}^{T}\lambda_j\hat{\mathcal{L}}(h,S_j)} \right ]}{\mathbb{E}_{h\sim P}\left [e^{-\sum_{j=1}^{T}\lambda_j\hat{\mathcal{L}}(h,S_j)} \right ]}
\end{split}
\end{align}
\end{theorem}

\begin{proof}
    By applying Theorem \ref{thm:forget-base2} with $T$ tasks, we have 
    \begin{align} \label{eq:forget-base-T}
\begin{split}
 \mathcal{L}&(Q_{1:T}, \mathcal{D}_t) \leq \hat{\mathcal{L}}(Q_{1:T}, S_T)+ \frac{1}{\lambda_T} D_{\mathrm{KL}}(Q_{1:T}||Q_{1:T-1})
 +\frac{1}{\lambda_T}\log\mathbb{E}_{h\sim Q_{1:T-1}}\left [e^{\lambda_T(\mathcal{L}(h,\mathcal{D}_t)-\hat{\mathcal{L}}(h,S_T))} \right ]
 \end{split}
 \end{align}
, suppose we assume that for each task we apply an empirical Gibbs learner, meaning $$\forall i\in[2,T], \hat{Q}^{\lambda_i}_{1:i}(h)\propto \hat{Q}^{\lambda_{i-1}}_{1:i-1}(h)e^{-\lambda_i\hat{\mathcal{L}}(h,S_i)}$$ 
where $\hat{Q}^{\lambda_1}_{1:1}(h)\propto P(h)e^{-\lambda_1\hat{\mathcal{L}}(h,S_1)}$.

We can provide bounds on the forgetting of $\hat{Q}^{\lambda_T}_{1:T}$ by setting the posterior $Q_{1:T}$ and the prior $Q_{1:T-1}$ as Gibbs posteriors:

\begin{align*} 
\begin{split}
\forall i\in[1,T-1]:
\mathcal{L}(\hat{Q}^{\lambda_T}_{1:T}, \mathcal{D}_i) &\leq \frac{1}{\lambda_T}\log\frac{\mathbb{E}_{h\sim \hat{Q}^{\lambda_{T-1}}_{1:T-1}}\left [e^{\lambda_T(\mathcal{L}(h,\mathcal{D}_i)-\hat{\mathcal{L}}(h,S_T))} \right ]}{\mathbb{E}_{h\sim \hat{Q}^{\lambda_{T-1}}_{1:T-1}}\left [e^{-\lambda_T\hat{\mathcal{L}}(h,S_T)} \right ]}
\end{split}
\end{align*}

By using the definition of the Gibbs posterior, we can unravel the expectations:

\begin{align*} 
\begin{split}
\forall i\in[1,T-1]:
\mathcal{L}(\hat{Q}^{\lambda_T}_{1:T}, \mathcal{D}_i) &\leq \frac{1}{\lambda_T}\log\frac{\mathbb{E}_{h\sim \hat{Q}^{\lambda_{T-2}}_{1:T-2}}\left [e^{\lambda_T(\mathcal{L}(h,\mathcal{D}_i)-\hat{\mathcal{L}}(h,S_T))-\lambda_{T-1}\mathcal{L}(h,\mathcal{D}_{T-1})} \right ]/Z_{T-1}}{\mathbb{E}_{h\sim \hat{Q}^{\lambda_{T-2}}_{1:T-2}}\left [e^{-\lambda_{T-1}\mathcal{L}(h,\mathcal{D}_{T-1})-\lambda_T\hat{\mathcal{L}}(h,S_T)} \right ]/Z_{T-1}}
\end{split}
\end{align*}

Setting $\lambda_{T-1}=\lambda_T$ and repeating this process, we arrive at:

\begin{align*} 
\begin{split}
\forall i\in[1,T-1]:
\mathcal{L}(\hat{Q}^{\lambda_T}_{1:T}, \mathcal{D}_i) &\leq \frac{1}{\lambda_T}\log\frac{\mathbb{E}_{h\sim P}\left [e^{\lambda_T\mathcal{L}(h,\mathcal{D}_i)-\sum_{j=1}^{T}\lambda_j\hat{\mathcal{L}}(h,S_j)} \right ]}{\mathbb{E}_{h\sim P}\left [e^{-\sum_{j=1}^{T}\lambda_j\hat{\mathcal{L}}(h,S_j)} \right ]}
\end{split}
\end{align*}

Taking an expectation over $S_i$,

\begin{align*} 
\begin{split}
\forall i\in[1,T-1]:
\mathbb{E}_{S_i}\mathcal{L}(\hat{Q}^{\lambda_T}_{1:T}, \mathcal{D}_i) &\leq \frac{1}{\lambda_T}\mathbb{E}_{S_i}\log\frac{\mathbb{E}_{h\sim P}\left [e^{\lambda_T\mathcal{L}(h,\mathcal{D}_i)-\sum_{j=1}^{T}\lambda_j\hat{\mathcal{L}}(h,S_j)} \right ]}{\mathbb{E}_{h\sim P}\left [e^{-\sum_{j=1}^{T}\lambda_j\hat{\mathcal{L}}(h,S_j)} \right ]}
\end{split}
\end{align*}

Applying Jensen's inequality:

\begin{align*} 
\begin{split}
\forall i\in[1,T-1]:
&\mathbb{E}_{S_i}\mathcal{L}(\hat{Q}^{\lambda_T}_{1:T}, \mathcal{D}_i) \\&\leq \frac{1}{\lambda_T}\mathbb{E}_{S_i}\log\frac{\mathbb{E}_{h\sim P}\left [\mathbb{E}_{S_i}\left [e^{\lambda_T\mathcal{L}(h,\mathcal{D}_i)-\lambda_i\hat{\mathcal{L}}(h,S_i)} \right ]e^{-\sum_{j=1,j\neq i}^{T}\lambda_j\hat{\mathcal{L}}(h,S_j)} \right ]}{\mathbb{E}_{h\sim P}\left [e^{-\sum_{j=1}^{T}\lambda_j\hat{\mathcal{L}}(h,S_j)} \right ]}
\end{split}
\end{align*}

If $\lambda_i=\lambda_T$, we can apply Hoeffding's lemma:

\begin{align*} 
\begin{split}
\forall i\in[1,T-1]:
\mathbb{E}_{S_i}\mathcal{L}(\hat{Q}^{\lambda_T}_{1:T}, \mathcal{D}_i) &\leq \frac{\lambda_T K^2}{8m_i}+\frac{1}{\lambda_T}\mathbb{E}_{S_i}\log\frac{\mathbb{E}_{h\sim P}\left [e^{-\sum_{j=1,j\neq i}^{T}\lambda_j\hat{\mathcal{L}}(h,S_j)} \right ]}{\mathbb{E}_{h\sim P}\left [e^{-\sum_{j=1}^{T}\lambda_j\hat{\mathcal{L}}(h,S_j)} \right ]}
\end{split}
\end{align*}

\end{proof}

\begin{theorem} Restatement of Theorem \ref{thm:forgetting-asymptotic}
For any $\lambda_T>0$, assuming all $Q_{1:j}$ are empirical Gibbs posteriors, and that Assumption \ref{assume:laplace-things} holds, the following holds a.s. in expectation over $S_{j}\sim \mathcal{D}_j$, $\forall j\in[T]$
\begin{align} 
\begin{split}
\lim_{m,T\rightarrow \infty}\frac{1}{T-1}\mathbb{E}_{S_1,\ldots,S_T}\sum_{i=1}^{T-1}\mathcal{L}(\hat{Q}^{\lambda_T}_{1:T}, \mathcal{D}_i) &\leq \lim_{m,T\rightarrow \infty}\frac{1}{T-1}\sum_{i=1}^{T-1}\mathcal{L}(h^*_{1:T/i}, \mathcal{D}_i).
\end{split}
\end{align}
\end{theorem}

\begin{proof}
    Starting from \eqref{eq:oracle-gibbs-ratio}, we have (assuming all training sets are of size $m$ and $\lambda_i=\lambda_T$) $\forall i\in[1,T-1]$
    
\begin{align*} 
\begin{split}
\mathbb{E}_{S_i}\mathcal{L}(\hat{Q}^{\lambda_T}_{1:T}, \mathcal{D}_i) &\leq \frac{\lambda_T K^2}{8m}+\frac{1}{\lambda_T}\mathbb{E}_{S_i}\log\frac{\mathbb{E}_{h\sim P}\left [e^{-\lambda_T\sum_{j=1,j\neq i}^{T}(\hat{\mathcal{L}}(h,S_j)-\mathcal{L}(h,\mathcal{D}_j)+\mathcal{L}(h,\mathcal{D}_j))} \right ]}{\mathbb{E}_{h\sim P}\left [e^{-\lambda_T\sum_{j=1}^{T}(\hat{\mathcal{L}}(h,S_j)-\mathcal{L}(h,\mathcal{D}_j)+\mathcal{L}(h,\mathcal{D}_j))} \right ]}
\end{split}
\end{align*}

Taking the expectation on all training sets (marked $\mathbb{E}_S$ for brevity), we can apply Jensen's inequality as well as Fubini's Theorem to obtain $\forall i\in[1,T-1]$

\begin{align*} 
\begin{split}
\mathbb{E}_{S}\mathcal{L}(\hat{Q}^{\lambda_T}_{1:T}, \mathcal{D}_i) &\leq \frac{\lambda_T K^2}{8m_i}+\frac{1}{\lambda_T}\log\mathbb{E}_{S}\frac{\mathbb{E}_{h\sim P}\left [\mathbb{E}_{S}\left [e^{-\lambda_T\sum_{j=1,j\neq i}^{T}(\hat{\mathcal{L}}(h,S_j)-\mathcal{L}(h,\mathcal{D}_j))}\right ]e^{-\lambda_T\sum_{j=1,j\neq i}^{T}\mathcal{L}(h,\mathcal{D}_j)} \right ]}{\mathbb{E}_{h\sim P}\left [e^{-\lambda_T\sum_{j=1}^{T}(\hat{\mathcal{L}}(h,S_j)-\mathcal{L}(h,\mathcal{D}_j)+\mathcal{L}(h,\mathcal{D}_j))} \right ]}
\end{split}
\end{align*}

Via a similar application of Hoeffding's lemma as in Lemma \ref{lemma:hoeffding-concentration}, we have (for $m\rightarrow\infty$)

\begin{align*} 
\begin{split}
\forall i\in[1,T-1]:
\mathbb{E}_{S}\mathcal{L}(\hat{Q}^{\lambda_T}_{1:T}, \mathcal{D}_i) &\leq \frac{\lambda_T K^2}{8m_i}+\frac{\lambda_T K^2 T}{8m}+\frac{1}{\lambda_T}\log\mathbb{E}_{S}\frac{\mathbb{E}_{h\sim P}\left [e^{-\lambda_T\sum_{j=1,j\neq i}^{T}\mathcal{L}(h,\mathcal{D}_j)} \right ]}{\mathbb{E}_{h\sim P}\left [e^{-\lambda_T\sum_{j=1}^{T}(\hat{\mathcal{L}}(h,S_j)-\mathcal{L}(h,\mathcal{D}_j)+\mathcal{L}(h,\mathcal{D}_j))} \right ]}
\end{split}
\end{align*}

If we have (for $m\rightarrow\infty$) $$cov_P\left (e^{\lambda_T\sum_{j=1}^{T}\mathcal{L}(h,\mathcal{D}_j)},e^{-\lambda_T\sum_{j=1}^{T}\hat{\mathcal{L}}(h,S_j)}\right )\leq 0$$ we also have
$$cov_P\left (e^{-\lambda_T\sum_{j=1}^{T}\mathcal{L}(h,\mathcal{D}_j)},e^{-\lambda_T\sum_{j=1}^{T}\hat{\mathcal{L}}(h,S_j)+\mathcal{L}(h,\mathcal{D}_j))}\right )\geq 0$$
and therefore (for $m\rightarrow\infty$)

\begin{align*} 
\begin{split}
\forall i\in[1,T-1]:
\mathbb{E}_{S}\mathcal{L}(\hat{Q}^{\lambda_T}_{1:T}, \mathcal{D}_i) &\leq \ldots+\frac{1}{\lambda_T}\log\mathbb{E}_{S}\frac{\mathbb{E}_{h\sim P}\left [e^{-\lambda_T\sum_{j=1,j\neq i}^{T}\mathcal{L}(h,\mathcal{D}_j)} \right ]}{\mathbb{E}_{h\sim P}\left [e^{-\lambda_T\sum_{j=1}^{T}\hat{\mathcal{L}}(h,S_j)-\mathcal{L}(h,\mathcal{D}_j)} \right ]\mathbb{E}_{h\sim P}\left [e^{-\lambda_T\sum_{j=1}^{T}\mathcal{L}(h,\mathcal{D}_j)} \right ]}
\end{split}
\end{align*}

\begin{align*} 
\begin{split}
=\ldots+\frac{1}{\lambda_T}\log\frac{\mathbb{E}_{h\sim P}\left [e^{-\lambda_T\sum_{j=1,j\neq i}^{T}\mathcal{L}(h,\mathcal{D}_j)} \right ]}{\mathbb{E}_{h\sim P}\left [e^{-\lambda_T\sum_{j=1}^{T}\mathcal{L}(h,\mathcal{D}_j)} \right ]}\mathbb{E}_{S}\frac{1}{\mathbb{E}_{h\sim P}\left [e^{-\lambda_T\sum_{j=1}^{T}\hat{\mathcal{L}}(h,S_j)-\mathcal{L}(h,\mathcal{D}_j)} \right ]}
\end{split}
\end{align*}

Via Jensen's inequality 

\begin{align*} 
\begin{split}
\leq \frac{\lambda_T K^2}{8m_i}+\frac{\lambda_T K^2 T}{8m}+\frac{1}{\lambda_T}\log\frac{\mathbb{E}_{h\sim P}\left [e^{-\lambda_T\sum_{j=1,j\neq i}^{T}\mathcal{L}(h,\mathcal{D}_j)} \right ]}{\mathbb{E}_{h\sim P}\left [e^{-\lambda_T\sum_{j=1}^{T}\mathcal{L}(h,\mathcal{D}_j)} \right ]}\mathbb{E}_{S}\mathbb{E}_{h\sim P}\left [e^{\lambda_T\sum_{j=1}^{T}\hat{\mathcal{L}}(h,S_j)+\mathcal{L}(h,\mathcal{D}_j)} \right ]
\end{split}
\end{align*}

Via Fubini's theorem and Hoeffding's lemma, similarly to Lemma \ref{lemma:hoeffding-concentration}, 

\begin{align*} 
\begin{split}
\forall i\in[1,T-1]:
\mathbb{E}_{S}\mathcal{L}(\hat{Q}^{\lambda_T}_{1:T}, \mathcal{D}_i) &\leq \frac{\lambda_T K^2}{8m_i}+\frac{\lambda_T K^2 T}{4m}+\frac{1}{\lambda_T}\log\frac{\mathbb{E}_{h\sim P}\left [e^{-\lambda_T\sum_{j=1,j\neq i}^{T}\mathcal{L}(h,\mathcal{D}_j)} \right ]}{\mathbb{E}_{h\sim P}\left [e^{-\lambda_T\sum_{j=1}^{T}\mathcal{L}(h,\mathcal{D}_j)} \right ]}
\end{split}
\end{align*}

Averaging over tasks, we have

\begin{align*} 
    \begin{split}
\frac{1}{T-1}\mathbb{E}_{S}\sum_{i=1}^{T-1}\mathcal{L}(\hat{Q}^{\lambda_T}_{1:T}, \mathcal{D}_i) &\leq \frac{\lambda_T K^2}{8m}+\frac{\lambda_T K^2 T}{4m}+\\&\frac{1}{\lambda_T(T-1)}\mathbb{E}_{S}\sum_{i=1}^{T-1}\log\frac{\mathbb{E}_{h\sim P}\left [e^{-\lambda_T(T-1)\sum_{j=1,j\neq i}^{T}\frac{\mathcal{L}(h,\mathcal{D}_j)}{(T-1)}} \right ]}{\mathbb{E}_{h\sim P}\left [e^{-\lambda_T \mathcal{L}(h,\mathcal{D}_i)}e^{-\lambda_T(T-1)\sum_{j=1,j\neq i}^{T}\frac{\mathcal{L}(h,\mathcal{D}_j)}{(T-1)}}\right ]}
\end{split}
\end{align*}
Under Assumption \ref{assume:laplace-things}, we can apply Laplace's method \citep{shun1995laplace} on both numerator and denominator assuming $\lambda_T (T-1)\rightarrow \infty$. Taking $\lambda_T=\sqrt{m}/T$ and considering the limit $m,T \rightarrow \infty$ we therefore have 
\begin{align*} 
    \begin{split}
&\lim_{m,T\rightarrow \infty}\frac{1}{T-1}\mathbb{E}_{S}\sum_{i=1}^{T-1}\mathcal{L}(\hat{Q}^{\sqrt{m}/T}_{1:T}, \mathcal{D}_i) \leq \lim_{m,T\rightarrow \infty}\frac{T}{\sqrt{m}(T-1)}\mathbb{E}_{S}\sum_{i=1}^{T-1}\log A_i
\end{split}
\end{align*}
where
\begin{equation*}
    A_i=\frac{C(m,T,d,h^*_{1:T/i})\cdot\left [e^{-\sqrt{m}(T-1)/T\sum_{j=1,j\neq i}^{T}\frac{\mathcal{L}(h^*_{1:T/i},\mathcal{D}_j)}{(T-1)}} \right ]\left (1+O(\frac{T}{\sqrt{m}T})+\ldots\right )}{C(m,T,d,h^*_{1:T/i})\cdot\left [e^{-\sqrt{m}/T\mathcal{L}(h^*_{1:T/i},\mathcal{D}_i)}e^{-\sqrt{m}(T-1)/T\sum_{j=1,j\neq i}^{T}\frac{\mathcal{L}(h^*_{1:T/i},\mathcal{D}_j)}{(T-1)}} \right ]\left (1+O(\frac{T}{\sqrt{m}T})+\ldots\right )}.
\end{equation*}

Since most of the elements in the numerator and denominator are shared, we have 
\begin{align*} 
    \begin{split}
&\lim_{m,T\rightarrow \infty}\frac{1}{T-1}\mathbb{E}_{S}\sum_{i=1}^{T-1}\mathcal{L}(\hat{Q}^{\sqrt{m}/T}_{1:T}, \mathcal{D}_i) \leq \lim_{m,T\rightarrow \infty}\frac{T}{\sqrt{m}(T-1)}\mathbb{E}_{S}\sum_{i=1}^{T-1}\log\frac{1}{e^{-\sqrt{m}/T\hat{\mathcal{L}}(h^*_{1:T/i},S_i)}}
\end{split}
\end{align*}

A rearrangement of terms completes the proof.
\end{proof}

\begin{theorem} Restatement of Theorem \ref{thm:forgetting-extended}:
For any $\lambda_T>0$, assuming all $Q_{1:j}$ are empirical Gibbs posteriors, $\ell\in[0,K]$, and that
 $\mathrm{cov}_{P}(e^{-\lambda_T\hat{\mathcal{L}}(h,S_i)}, e^{-\sum_{j=1,j\neq i}^{T}\lambda_j\hat{\mathcal{L}}(h,S_j)})\geq 0$,
for any sample of training sets $S_{j}\sim \mathcal{D}_j$,
\begin{align} 
\begin{split}
\forall i\in[1,T]:
\mathbb{E}_{S_i}\mathcal{L}(\hat{Q}^{\lambda_T}_{1:T}, \mathcal{D}_i) &\leq \frac{\lambda_T K^2}{8m_i}+\mathcal{L}(P,\mathcal{D}_i)
\end{split}
\end{align}
\end{theorem}

\begin{proof}
    Starting from \eqref{eq:oracle-gibbs-ratio}, we mark 
$$\mathrm{cov}_{P}(i, [T])\triangleq\mathrm{cov}_{P}(e^{-\lambda_T\hat{\mathcal{L}}(h,S_i)}, e^{-\sum_{j=1,j\neq i}^{T}\lambda_j\hat{\mathcal{L}}(h,S_j)}).$$

From the definition of covariance, we can decompose the denominator:

\begin{align} \label{eq-oracle-forget-extend}
\begin{split}
&\forall i\in[1,T-1]:
\mathbb{E}_{S_i}\mathcal{L}(\hat{Q}^{\lambda_T}_{1:T}, \mathcal{D}_i) \leq \frac{\lambda_T K^2}{8m_i}\\&+\frac{1}{\lambda_T}\mathbb{E}_{S_i}\log\frac{\mathbb{E}_{h\sim P}\left [e^{-\sum_{j=1,j\neq i}^{T}\lambda_j\hat{\mathcal{L}}(h,S_j)} \right ]}{\mathbb{E}_{h\sim P}\left [e^{-\sum_{j=1,j\neq i}^{T}\lambda_j\hat{\mathcal{L}}(h,S_j)} \right ]\mathbb{E}_{h\sim P}\left [e^{-\lambda_T\hat{\mathcal{L}}(h,S_i)} \right ]+\mathrm{cov}_{P}(i, [T])}
\end{split}
\end{align}

And if this covariance is non-negative, we can set it at $0$ and still retain a valid upper bound, giving us the forgetting bound via Jensen's inequality.

Similarly for forward transfer, we get 

\begin{align} \label{eq-oracle-forward-extend}
\begin{split}
&\mathbb{E}_{S_T}\mathcal{L}(\hat{Q}^{\lambda_T}_{1:T}, \mathcal{D}_T) \leq \frac{\lambda_T K^2}{8m_T}\\&+
\frac{1}{\lambda_T}\mathbb{E}_{S_T}\log\frac{\mathbb{E}_{h\sim P}\left [e^{-\sum_{j=1}^{T-1}\lambda_j\hat{\mathcal{L}}(h,S_j)} \right ]}{\mathbb{E}_{h\sim P}\left [e^{-\sum_{j=1}^{T-1}\lambda_j\hat{\mathcal{L}}(h,S_j)} \right ]\mathbb{E}_{h\sim P}\left [e^{-\lambda_T\hat{\mathcal{L}}(h,S_T)} \right ]+\mathrm{cov}_{P}(T, [T])}
\end{split}
\end{align}

and apply $\mathrm{cov}_{P}(T, [T])\geq 0$ and Jensen's inequality to get the bound for generalization.

\end{proof}


\begin{corollary} \label{thm:oracle-T-highcov}:
    Appendix-only Corollary: Under the same conditions as Theorem \ref{thm:forgetting-extended}, if we additionally have
\begin{align*} 
\begin{split}
\mathrm{cov}_{P}(e^{-\lambda_T\hat{\mathcal{L}}(h,S_i)}, e^{-\sum_{j=1,j\neq i}^{T}\lambda_j\hat{\mathcal{L}}(h,S_j)})
 \geq e^{-c}-\mathbb{E}_{h\sim P}\left [e^{-\lambda_T\hat{\mathcal{L}}(h,S_i)} \right ],
\end{split}
\end{align*}
we have (for any sample of training sets $S_{j}\sim \mathcal{D}_j$),
\begin{align*} 
\begin{split}
\forall i\in[1,T-1]:
\mathbb{E}_{S_i}\mathcal{L}(\hat{Q}^{\lambda_T}_{1:T}, \mathcal{D}_i) &\leq \frac{\lambda_T K^2}{8m_i}+\frac{c}{\lambda_T}.
\end{split}
\end{align*}
\end{corollary}

\begin{proof}
Considering our forgetting bound \eqref{eq-oracle-forget-extend}, we consider when 

\begin{align*} 
\begin{split}
\frac{\mathbb{E}_{h\sim P}\left [e^{-\sum_{j=1,j\neq i}^{T}\lambda_j\hat{\mathcal{L}}(h,S_j)} \right ]}{\mathbb{E}_{h\sim P}\left [e^{-\sum_{j=1,j\neq i}^{T}\lambda_j\hat{\mathcal{L}}(h,S_j)} \right ]\mathbb{E}_{h\sim P}\left [e^{-\lambda_T\hat{\mathcal{L}}(h,S_i)} \right ]+\mathrm{cov}_{P}(i, [T])} \leq e^{c},
\end{split}
\end{align*}

Moving terms around, this condition is satisfied if

\begin{align*} 
\begin{split}
\frac{\mathrm{cov}_{P}(i, [T])}{\mathbb{E}_{h\sim P}\left [e^{-\sum_{j=1,j\neq i}^{T}\lambda_j\hat{\mathcal{L}}(h,S_j)} \right ]}+\mathbb{E}_{h\sim P}\left [e^{-\lambda_T\hat{\mathcal{L}}(h,S_i)} \right ]
 \geq e^{-c},
\end{split}
\end{align*}

We note that since $\ell\in[0,K]$, we have $0<\mathbb{E}_{h\sim P}\left [e^{-\sum_{j=1,j\neq i}^{T}\lambda_j\hat{\mathcal{L}}(h,S_j)} \right ]\leq 1.$

We can further simplify this by taking a slightly looser condition:

\begin{align*} 
\begin{split}
\mathrm{cov}_{P}(i, [T])
 \geq e^{-c}-\mathbb{E}_{h\sim P}\left [e^{-\lambda_T\hat{\mathcal{L}}(h,S_i)} \right ]
\end{split}
\end{align*}

and this is the condition we assumed. As such, we can replace the last term in \eqref{eq-oracle-forget-extend} with 

$$\frac{1}{\lambda_t}\log e^{c},$$
giving us a bound of the form:

\begin{align} \label{eq-oracle-forget-extend-best}
\begin{split}
\forall i\in[1,T-1]:
\mathbb{E}_{S_i}\mathcal{L}(\hat{Q}^{\lambda_T}_{1:T}, \mathcal{D}_i) &\leq \frac{\lambda_T K^2}{8m_i}+\frac{c}{\lambda_T}.
\end{split}
\end{align}

\end{proof}


\begin{theorem} Restatement of \ref{thm:forgetting-asymptotic-cov}
Assuming all $Q_{1:j}$ are empirical Gibbs posteriors with $\lambda_i=\sqrt{m}$, $P$ is a uniform prior over $\mathcal{H}$ , and that Assumption \ref{assume:laplace-things} holds with a condition on $h^*_i=\argmin_h\mathcal{L}(h,\mathcal{D}_i)$ instead of the condition on all other other tasks, 
if $\ \forall i\in[T-1],\  \mathrm{cov}_{P}(i, [T])\geq 0$,
the following holds a.s. in expectation over $S_{j}\sim \mathcal{D}_j$, $\forall j\in[T]$
\begin{align} 
\begin{split}
\lim_{m\rightarrow \infty}\frac{1}{T-1}\mathbb{E}_{S_1,\ldots,S_T}\sum_{i=1}^{T-1}\mathcal{L}(\hat{Q}^{\sqrt{m}}_{1:T}, \mathcal{D}_i) &\leq \frac{1}{T-1}\sum_{i=1}^{T-1}\mathcal{L}(h^*_{i}, \mathcal{D}_i)+\lim_{m \rightarrow\infty}\frac{\mathbb{E}_{S}\sum_{i=1}^{T-1}\log \det \mathcal{L}_i''(h^*_{i})}{2\sqrt{m}(T-1)}.
\end{split}
\end{align}
\end{theorem}

\begin{proof}
    Starting from \eqref{eq-oracle-forget-extend}, we can use the fact that $\mathrm{cov}_{P}(i, [T])\geq 0$ to upper bound
    \begin{align*} 
\begin{split}
&\mathbb{E}_{S_i}\mathcal{L}(\hat{Q}^{\lambda_T}_{1:T}, \mathcal{D}_i) \leq \frac{\lambda_T K^2}{8m_i}\\&+
\frac{1}{\lambda_T}\mathbb{E}_{S_T}\log\frac{\mathbb{E}_{h\sim P}\left [e^{-\sum_{j=1,j\neq i}^{T}\lambda_j\hat{\mathcal{L}}(h,S_j)} \right ]}{\mathbb{E}_{h\sim P}\left [e^{-\sum_{j=1,j\neq i}^{T}\lambda_j\hat{\mathcal{L}}(h,S_j)} \right ]\mathbb{E}_{h\sim P}\left [e^{-\lambda_T\hat{\mathcal{L}}(h,S_i)} \right ]+0}
\end{split}
\end{align*}
giving us the simplified
    \begin{align*} 
\begin{split}
&\mathbb{E}_{S_i}\mathcal{L}(\hat{Q}^{\lambda_T}_{1:T}, \mathcal{D}_i) \leq \frac{\lambda_T K^2}{8m_i}\\&+
\frac{1}{\lambda_T}\mathbb{E}_{S_T}\log\frac{1}{\mathbb{E}_{h\sim P}\left [e^{-\lambda_T\hat{\mathcal{L}}(h,S_i)} \right ]}
\end{split}
\end{align*}

Similarly to Theorem \ref{thm:forgetting-asymptotic}, we can apply Jensen's inequality and Hoeffding's lemma, yielding

\begin{align*} 
\begin{split}
&\mathbb{E}_{S}\mathcal{L}(\hat{Q}^{\lambda_T}_{1:T}, \mathcal{D}_i) \leq \frac{\lambda_T K^2}{4m_i}\\&+
\frac{1}{\lambda_T}\mathbb{E}_{S_T}\log\frac{1}{\mathbb{E}_{h\sim P}\left [e^{-\lambda_T\mathcal{L}(h,\mathcal{D}_i)} \right ]}
\end{split}
\end{align*}

Averaging out these terms, we can (under Assumption \ref{assume:laplace-things}) apply Laplace's method on the denominator assuming  $\lambda_T=\sqrt{m}\rightarrow \infty$, meaning
\begin{align*} 
    \begin{split}
&\lim_{m\rightarrow \infty}\frac{1}{T-1}\mathbb{E}_{S}\sum_{i=1}^{T-1}\mathcal{L}(\hat{Q}^{\sqrt{m}}_{1:T}, \mathcal{D}_i) \leq \lim_{m \rightarrow\infty}\frac{1}{\sqrt{m}(T-1)}\mathbb{E}_{S}\sum_{i=1}^{T-1}\log \frac{1}{\left (\frac{2\pi}{\sqrt{m}}\right )^{d/2}\frac{P(h^*_i)}{\sqrt{\det \mathcal{L}_i''(h^*_{i})}}\left [e^{-\sqrt{m}\mathcal{L}(h^*_{i},\mathcal{D}_i)} \right ]}
\end{split}
\end{align*}
Considering the logarithm we can arrive at
\begin{align*} 
    \begin{split}
&\lim_{m\rightarrow \infty}\frac{1}{T-1}\mathbb{E}_{S}\sum_{i=1}^{T-1}\mathcal{L}(\hat{Q}^{\sqrt{m}}_{1:T}, \mathcal{D}_i) \leq \lim_{m \rightarrow\infty}\frac{1}{\sqrt{m}(T-1)}\mathbb{E}_{S}\sum_{i=1}^{T-1}\log \left (\left (\frac{2\pi}{\sqrt{m}}\right )^{-d/2}\frac{\sqrt{\det \mathcal{L}_i''(h^*_{i})}}{P(h^*_i)}\left [e^{\sqrt{m}\mathcal{L}(h^*_{i},\mathcal{D}_i)} \right ]\right )
\end{split}
\end{align*}
Since $m\rightarrow \infty$, we have 
\begin{align*} 
    \begin{split}
&\lim_{m\rightarrow \infty}\frac{1}{T-1}\mathbb{E}_{S}\sum_{i=1}^{T-1}\mathcal{L}(\hat{Q}^{\sqrt{m}}_{1:T}, \mathcal{D}_i) \leq \lim_{m \rightarrow\infty}\frac{\mathbb{E}_{S}\sum_{i=1}^{T-1}\log \det \mathcal{L}_i''(h^*_{i})}{2\sqrt{m}(T-1)}+\frac{1}{(T-1)}\sum_{i=1}^{T-1}\mathcal{L}(h^*_{i},\mathcal{D}_i)
\end{split}
\end{align*}

\end{proof}

\newpage
\section{Appendix - configuration and hyper-parameters} \label{append:hyperparam}

For $10d$-Gaussian tasks, we use several setups of binary classification tasks in $\mathbb{R}^{10}$. All tasks draw samples from a $10$-dimensional Gaussian distribution $p(x) = \mathcal{N}(x;0,I_{10})$, and $y=\mathrm{sgn}(a^\top x)$. 
We consider several settings for $a$, where for all setting, all values of $a$ other than the first two are zero, meaning the last $8$ features do not impact the label.
We consider the following settings:
\begin{itemize}
    \item Similar tasks: linear separators for all tasks are within $10^\circ$ of some reference angle.
    \item Gradual shift: tasks change angle gradually in a set direction, with each consecutive task being within $10^\circ$ of the last one.
    \item Orthogonal shift: two sets of similar tasks with a $90^\circ$ angle between the separators of both task sets. The first half of tasks are of the first type and the second half corresponds to the second type.
\end{itemize}

For $10d$-Gaussian tasks, we use $3000$ samples per task, and a total of $T=100$ tasks.

\begin{table*}[b]
\caption{Discounted ($\gamma=0.95$) backward transfer and forgetting for $10d$ continual tasks at $T=100$. Metrics are in loss percentage over the relevant models and test sets averaged over $5$ seeds. Standard error is reported after the $\pm$ sign.} 
\label{2d-discounts-table}
\vskip 0.1in
\begin{center}
\begin{small}
\begin{sc}
\begin{tabular}{lcccc}
Method  & BWT & BWT bound   & Forgetting &  Forgetting bound   \\
\midrule
 Similar tasks   \\
\midrule
EWC    & $0.4 \pm 0.0$ & $0.4 \pm 0.0$ & $0.1 \pm 0.0$ & $0.1 \pm 0.0$  \\
VI & $0.3 \pm 0.1$ & $0.3 \pm 0.1$ & $-0.0 \pm 0.0$ & $-0.0 \pm 0.0$ \\
\midrule
 Gradual shift   \\
 \midrule
EWC    & $0.4 \pm 0.0$ & $0.4 \pm 0.0$ & $0.0 \pm 0.0$ & $0.1 \pm 0.0$ \\
VI & $0.0 \pm 0.0$ & $0.1 \pm 0.0$ & $-0.0 \pm 0.0$ & $-0.0 \pm 0.0$ \\
\midrule
 Orth. shift   \\
 \midrule
EWC    & $1.0 \pm 0.2$ & $1.0 \pm 0.2$ & $0.1 \pm 0.0$ & $0.1 \pm 0.0$  \\
VI & $0.3 \pm 0.1$ & $0.3 \pm 0.1$ & $0.0 \pm 0.0$ & $0.1 \pm 0.0$ \\
\end{tabular}
\end{sc}
\end{small}
\end{center}
\vskip -0.1in
\end{table*}

The model consists of a shared fully connected layer of $64$ neurons and a tanh activation, as well as a linear classification head for each task. Each task was trained for $1$ epochs with batch size $16$. The learning rate was static at $1e^{-3}$.
The $\lambda$ parameter was set to $1e^{-3}$. 
All results were run for $5$ random seeds and averages and standard error were reported.

For EWC, the $\sigma^2$ parameter used for posterior construction was set as $1e^{-2}$, and the regularization weight was set at $\lambda_{\mathrm{EWC}}=40$.

For permuted-MNIST, we used a different pixel permutation per task, and each task involved $10$-way classification. All $60000$ training samples were used for training with a batch size of $128$. The learning rate was static at $1e^{-3}$ and the $\lambda$ parameter was set to $1e^{-3}$.
The model consists of a shared convolutional neural network consisting of 2 CNN blocks each comprised of $64$ $5\times 5$ $2d$ convolutions followed by max-pooling and tanh non-linearity, followed by a linear layer consisting of $400$ neurons and a tanh activation as well as a linear classification head for each task. 
All results were run for $5$ random seeds and averages and standard error were reported in Table \ref{permuted-mnist-full-table}.

\begin{table*}
\caption{Metrics for permuted-MNIST tasks. All test metrics are in loss percentage over the relevant models and test sets averaged over $5$ seeds. Backward transfer is the final loss percentage over all test sets. Standard error is reported after the $\pm$ sign.} 
\label{permuted-mnist-full-table}
\vskip 0.1in
\begin{center}
\begin{small}
\begin{sc}
\begin{tabular}{lccccc}
Method  &BWT &BWT bound  &Forgetting &Forgetting bound \eqref{eq:final-forget-bound} &Test loss \\
\midrule
EWC    & $1.7 \pm 0.1$ & $2.4 \pm 0.2$ & $0.6 \pm 0.1$ & $1.3 \pm 0.2$ & $1.0 \pm 0.1$  \\
VI & $10.4 \pm 0.3$ & $12.1 \pm 0.4$ & $-2.9 \pm 0.2$ & $-1.2 \pm 0.3$ & $3.1 \pm 0.1$ \\
\end{tabular}
\end{sc}
\end{small}
\end{center}
\vskip -0.1in
\end{table*}

For split-MNIST, each task involved half of the labels (at random) chosen as positive and half as negative. This is a minor departure from the standard split-MNIST problem where $5$ different binary classification tasks are created and their loss is averaged, but the overall behavior is similar.
All $60000$ training samples were used for training with a batch size of $128$. The learning rate was static at $1e^{-3}$ and the $\lambda$ parameter was set to $1e^{-3}$.
The model used is identical to the CNN used for permuted-MNIST.
All results were run for $5$ random seeds and averages and standard error were reported in Table \ref{split-mnist-full-table}.

\begin{table*}
\caption{Metrics for split-MNIST tasks. All test metrics are in loss percentage over the relevant models and test sets averaged over $5$ seeds. Backward transfer is the final loss percentage over all test sets. Standard error is reported after the $\pm$ sign.} 
\label{split-mnist-full-table}
\vskip 0.1in
\begin{center}
\begin{small}
\begin{sc}
\begin{tabular}{lccccc}
Method  &BWT &BWT bound &Forgetting &Forgetting bound \eqref{eq:final-forget-bound} &Test loss \\
\midrule
EWC    & $12.7 \pm 0.7$ & $13.6 \pm 0.7$ & $11.9 \pm 0.7$ & $12.8 \pm 0.7$ & $0.7 \pm 0.1$  \\
VI & $15.1 \pm 0.3$ & $17.7 \pm 0.4$ & $2.5 \pm 0.4$ & $5.1 \pm 0.5$ & $3.1 \pm 0.3$ \\
\end{tabular}
\end{sc}
\end{small}
\end{center}
\vskip -0.1in
\end{table*}

For split-CIFAR10, each task involved half of the labels (at random) chosen as positive and half as negative. This is a minor departure from the standard split-CIFAR10 problem where $5$ different binary classification tasks are created and their loss is averaged, but the overall behavior is similar.
All $50000$ training samples were used for training with a batch size of $256$. The learning rate was static at $1e^{-3}$ and the $\lambda$ parameter was set to $1e^{-3}$.
The model used is identical to the CNN used for permuted-MNIST but with $800$ neurons in the linear layer.
All results were run for $5$ random seeds and averages and standard error were reported in Table \ref{split-cifar-full-table}.

\begin{table*}
\caption{Metrics for split-CIFAR10 tasks. All test metrics are in loss percentage over the relevant models and test sets averaged over $5$ seeds. Backward transfer is the final loss percentage over all test sets. Standard error is reported after the $\pm$ sign.} 
\label{split-cifar-full-table}
\vskip 0.1in
\begin{center}
\begin{small}
\begin{sc}
\begin{tabular}{lccccc}
Method  &BWT &BWT bound &Forgetting &Forgetting bound \eqref{eq:final-forget-bound} &Test loss \\
\midrule
EWC    & $36.9 \pm 0.4 $ & $40.7 \pm 0.5$ & $5.7 \pm 0.3$ & $9.5 \pm 0.6$ & $30.9 \pm 0.9$  \\
VI & $48.3 \pm 0.0$ & $50.3 \pm 0.2$ & $-1.2 \pm 0.1$ & $0.8 \pm 0.3$ & $49.1 \pm 0.4$ \\
\end{tabular}
\end{sc}
\end{small}
\end{center}
\vskip -0.1in
\end{table*}

All experiments were run on local hardware with an NVIDIA GeForce 4090 GPU and an Intel i9 CPU.

\end{document}